\documentclass{article}
\usepackage{amsmath}
\usepackage{amsfonts}
\usepackage{amssymb}
\usepackage{amsthm}
\usepackage[utf8]{inputenc}
\usepackage{graphicx}
\usepackage[T1]{fontenc}
\usepackage{color}
\usepackage[11pt]{extsizes}
\usepackage[top=2cm, bottom=2cm, left=2cm, right=2cm]{geometry}
\usepackage[colorlinks=true,linkcolor=blue]{hyperref}
\usepackage{systeme}
\usepackage{comment}
\usepackage{bm}
\usepackage{tabularray}
\usepackage{multirow}
\usepackage{appendix}
\usepackage{pgfplots}
\usepackage{authblk}
\usepackage[ruled, lined, linesnumbered, commentsnumbered, longend]{algorithm2e}

\newtheorem{theorem}{Theorem}[section]

\newtheorem{Proposition}[theorem]{Proposition}
\theoremstyle{definition}
\newtheorem{example}{Example}
\newtheorem{definition}[theorem]{Definition}
\newtheorem{remark}{Remark}[section]

\makeatletter
\newenvironment{countlist}[2][]
  {\begin{enumerate}[#1]
     \setcounter{countlist}{0}%
     \def\countname{#2}%
     \let\olditem\item
     \renewcommand{\item}{\stepcounter{countlist}\olditem}}
  {  \renewcommand{\@currentlabel}{\arabic{countlist}}%
     \label{\countname}%
   \end{enumerate}}
\makeatother

\DeclareGraphicsExtensions{.pdf,.png}

\newcommand{\q}[1]{``#1''}

\DeclareMathOperator{\diag}{diag}
\DeclareMathOperator{\sign}{sign}

\usepackage{enumitem,refcount}

\newcounter{countlist}

\title{A new Input Convex Neural Network with application to options pricing}

\author[1]{Vincent Lemaire}
\author[1]{Gilles Pagès}
\author[1,2]{Christian Yeo}
\affil[1]{\footnotesize Sorbonne Université, Laboratoire de Probabilités, Statistique et Modélisation, UMR 8001, 75005 Paris, France}
\affil[2]{\footnotesize Engie Global Markets, 1 place Samuel Champlain, 92400 Courbevoie, France}

\date{}

\numberwithin{equation}{section}

\pgfplotsset{compat=1.18}

\begin{document}
\maketitle

\begin{abstract}
We introduce a new class of neural networks designed to be convex functions of their inputs, leveraging the principle that any convex function can be represented as the supremum of the affine functions it dominates. These neural networks, inherently convex with respect to their inputs, are particularly well-suited for approximating the prices of options with convex payoffs. We detail the architecture of this, and establish theoretical convergence bounds that validate its approximation capabilities. We also introduce a \emph{scrambling} phase to improve the training of these networks. Finally, we demonstrate numerically the effectiveness of these networks in estimating prices for three types of options with convex payoffs: Basket, Bermudan, and Swing options.
\end{abstract}


\section*{Introduction}
Neural networks have become powerful tools for function and representation approximation. In many fields, such as finance, the target function to approximate is often convex, and preserving this convexity is essential, as failing to do so can lead to financial losses because of arbitrage opportunities. In this paper, we introduce a novel class of convex networks: neural networks specifically structured to yield convex functions of their inputs. It is worth noticing that, to this end, some approaches have already been developed. Among others, one may find \emph{Input Convex Neural Networks} (\emph{ICNN}) \cite{pmlr-v70-amos17b} and \emph{GroupMax} networks \cite{MR4791896}.

\emph{Input Convex Neural Networks} are neural networks where activation functions are convex, and non-decreasing functions. This design ensures the convexity of the output of the neural network with respect to the inputs as long as weights of the neural network are non-negative. Indeed, it is straightforward that the composition of a non-decreasing convex function with a convex function (actually affine function, representing layers' outputs) is still convex. In this method, the positiveness constraint on weights of the network can be met using penalization or using \q{\emph{passthrough}} layers as suggested in \cite{pmlr-v70-amos17b}. Besides, the representational power of \emph{ICNNs} is discussed in \cite{pmlr-v70-amos17b}, where, to strengthen the representational capacity of \emph{ICNNs}, it also proposed the use of \emph{Partially Input Convex Neural Networks (PICNN)} which outputs are convex functions only with respect to some of their inputs.

\emph{GroupMax} networks builds upon a well-known characterization of convex functions: It can be shown that a function $f : \mathbb{R}^d \to \mathbb{R}$ is convex if and only if it is the supremum of the affine functions it dominates. Thus, the authors designed a neural network architecture that progressively applies the maximum operation to groups of outputs across hidden layers. To ensure that the composition of maximum functions preserves convexity, they enforced non-negativity of weights using the function $\mathbb{R} \ni x \mapsto \max(x, 0)$. Furthermore, the authors established a \emph{universal approximation theorem}, demonstrating that their network can approximate any convex function on a compact domain to arbitrary precision.

In this paper, we propose a new \emph{Convex Network} (\emph{CN}), allowing to get rid of the positiveness constraint of weights aforementioned (be that for \emph{ICNN} or \emph{GroupMax} network), while still producing a convex function. Our architecture has the advantage to be very simple, yet effective. Besides, to strengthen the training phase of our network, we build upon the \emph{Deep Linear Network} literature \cite{JMLR:v25:23-0493, arora2018a, BahBubacarrLinearNetRiemman}.

\emph{Linear Neural Networks} are networks whose units have no activation functions. Thus, by construction, and as suggested by their names, they are designed to produce affine functions or, more equivalently, hyperplanes. In their simplest form, they are just one-layer neural networks without activation function. Their deep version obtained by stacking hidden layers without activation functions proved to have good optimization properties. In \cite{JMLR:v25:23-0493} (see also \cite{NIPS2016_f2fc9902}), the authors studied the loss landscapes in \emph{Deep Linear Neural Networks} and highlighted challenges like escaping from saddle points. They have showed that every critical point of these landscapes that is not a global minimum is a saddle point, meaning stochastic gradient-based methods can theoretically escape them (see \cite{Jin2017HowTE, ge2018learning}). Other works, such as \cite{arora2018a}, demonstrate that stochastic gradient descents in deep linear networks converge to a global minimum at a linear rate. Similarly, \cite{pmlr-v97-du19a} shows that with high probability, gradient descent with \emph{Xavier initialization} \cite{pmlr-v9-glorot10a} reaches a global minimum in sufficiently large deep linear networks.

Our network builds on the characterization of convex functions as the supremum of the affine functions they dominate. In this characterization, affine functions may be thought as layers' outputs, and the maximum function can be thought as the activation function of the output layer. Besides, the intuition behind the strength of our network is not only based on the characterization of convex functions aforementioned, but also on the fact that, from the \emph{Deep Linear Neural Network} literature, we know that stacking layers without activation functions has good optimization properties while producing affine functions (or hyperplanes). Thus, in our \emph{Convex Network}, the affine part can be replaced by a stack of hidden layers without activation functions, the maximum function serving as the activation function of the output layer. This yields what we will call \emph{Scrambling} and is closely related to \emph{attention mechanism} \cite{NIPS2017_3f5ee243}. Indeed, this idea can be compared to the first linear transformation in the attention mechanism, where input vectors are transformed into queries, keys, and values through linear layers.

Convex functions may have useful properties, such as having a single global minimum (when coercive), which makes our \emph{convex network} particularly suitable for applications where convexity is important, such as optimization problems, machine learning, and pricing models. The latter field is the main focus of this paper. We will demonstrate that our \emph{Convex Network} is an efficient alternative when it comes to the pricing of options or other financial instruments with convex payoffs.

Options pricing is a fundamental problem in Financial Mathematics, where the goal is to determine the fair value of financial derivatives based on their underlying assets. Many of these financial derivatives involve instruments with convex payoffs. Among others, one can cite options like basket \cite{DINGEC2013421}, Best-of-Bermudan \cite{Longstaff2001ValuingAO}, and Swing (or Take-or-Pay) options  \cite{Thompson1995ValuationOP}. The pricing of last two options relies on respectively \emph{Optimal Stopping Theory} and \emph{Stochastic Optimal Control theory}, whereas the pricing of the first option relies on classic Monte Carlo approach.

In this paper, we propose the use of our \emph{Convex Network} for the pricing of such options, a novel approach that leverages the inherent convexity of our \emph{Convex Network} architecture. Unlike traditional neural networks, which may struggle to maintain convexity properties (sometimes for non-arbitrability purposes) in pricing models, our \emph{Convex Network} guarantees that this essential condition is automatically met by design. Besides, as mentioned earlier, in this paper, we will tackle three types of options (Basket, Best-of-Bermudan, and Swing) which have convex payoffs. Especially, we consider these options since one can show that, under mild assumptions, their price is a convex function of the underlying asset price. While this is straightforward for basket options, it is more difficult for path-dependent options like Bermudan, and Swing options. In this tougher case, we will rely on \emph{convex ordering theory} \cite{Pagès2016_cvx_ord_path_dep, pagès2024convexorderingstochasticcontrol}. Then, our aim is to use our \emph{Convex Network} to approximate prices of such options as a function of the underlying asset value. For Bermudan, and Swing options, this will be based on the \emph{Backward Dynamic Programming Principle}, where the aim will be to approximate the \emph{continuation value} which is a conditional expectation. It is worth noting that approximating conditional expectations to solve such kind of problems has already been considered with classic feedforward neural networks (see \cite{MR4308650}).

The contributions of this paper are threefold. First, we establish theoretical convergence bounds, based on optimal quantization results (see \cite{bookGillesQuantifMarginale, Pagès2018Quantif}), for the approximation of convex functions using our \emph{Convex Network}, providing a solid foundation for their application in option pricing. Second, we implement our \emph{Convex Network} for each of the aforementioned option types and demonstrate its effectiveness through several numerical experiments. Third, we compare the performances of our \emph{Convex Network} against other state-of-the-art methods, highlighting its advantages in terms of accuracy.

Our paper is organized as follows. Section \ref{section_1} presents theoretical foundations of our \emph{Convex Network} as well as its approximation capacity, providing some error bounds. We also seize the opportunity to discuss different possible architectures for our \emph{Convex Network}. In Section \ref{section_2}, we prove the convexity of the prices of the three types of options aforementioned with respect to their underlying asset price. Section \ref{section_3} presents numerical results, attesting the effectiveness of our \emph{Convex Network}.

\vspace{0.4cm}
\noindent
\textbf{\textsl{Notations}.}
$\bullet$ $\mathbb{R}^d$ is equipped with the canonical Euclidean norm denoted by $|\cdot|$. For $x \in \mathbb{R}^d$ and $R > 0$, the closed ball of radius $R$ centered around $x$ is defined by:
$$\mathcal{B}(x, R) := \big\{y \in \mathbb{R}^d: \quad |x-y| \le R \big\}.$$

\noindent
$\bullet$ For all $x=(x_1, \ldots, x_d), y=(y_1,\ldots,y_d) \in \mathbb{R}^d$, $\langle x, y \rangle = \sum_{i = 1}^d x_iy_i$ denotes the Euclidean inner product. 

\noindent
$\bullet$ $\mathbb{M}_{d,q}(\mathbb{R})$ denotes the space of real-valued matrix with $d$ rows and $q$ columns and is equipped with the Fröbenius norm.

\noindent
$\bullet$ $\mathcal{S}^{+}\big(q, \mathbb{R}\big)$ and $\mathcal{O}\big(q, \mathbb{R}\big)$ denote respectively the subsets of $\mathbb{M}_{q, q}(\mathbb{R})$ of symmetric positive semi-definite and orthogonal matrices.

\noindent
$\bullet$ For $A \in \mathbb{M}_{d,d}(\mathbb{R})$, $\det(A)$ denotes the determinant of the matrix $A$.

\noindent
$\bullet$ For a sequence $(I_i)_{1\le i\le n}$ of sets, $\prod_{i=1}^n I_i$ denotes the cartesian product.

\noindent
$\bullet$ For a matrix $A \in \mathbb{M}_{d,q}(\mathbb{R})$, $A^\top$ denotes its transpose.

\noindent
$\bullet$ For $x := (x_1,\ldots,x_d) \in \mathbb{R}^d$, $\diag(x_1,\ldots,x_d)$ denotes the diagonal matrix in $\mathbb{M}_{d,d}(\mathbb{R})$ whose diagonal entries are made of component of vector $x$.

\section{Convex Neural Networks}
\label{section_1}
In this paper, we propose a new neural network architecture designed to produce convex function by nature. The rationale behind the construction of our \emph{convex network} lies in Proposition \ref{charac_convex_func_max_affine}, defining convex functions by the upper envelope of affine functions they dominate.

\begin{Proposition}[A characterization of convex functions]
\label{charac_convex_func_max_affine}
    \begin{itemize}
        \item Let $d \in \mathbb{N}^{*}$. A function $f:\mathbb{R}^d \to \mathbb{R}$ is convex if and only if for any $x \in \mathbb{R}^d $, one has:
        $$f(x) = \sup\Big\{\langle \alpha, x\rangle + \beta, \quad \alpha \in \mathbb{R}^d, \beta \in \mathbb{R} \quad \text{s.t.} \quad \forall y \in \mathbb{R}^d, \quad \langle \alpha, y\rangle + \beta \le f(y) \Big\}.$$

        \item In particular, for any sequence $(x_i)_{i \ge 1}$ everywhere dense in $\mathbb{R}^d$, one has for any $x \in \mathbb{R}^d$:
        $$\underset{i=1:n}{\max} \hspace{0.1cm} \varphi_i(x) := \underset{i=1:n}{\max} \hspace{0.1cm} \Big[\langle \nabla f(x_i), x-x_i\rangle + f(x_i)  \Big] \xrightarrow[n \to +\infty]{} f(x),$$

        \noindent
        where $\nabla f$ denotes the subgradient of the convex function $f$.
    \end{itemize}
\end{Proposition}

The preceding proposition suggests (see \cite{pmlr-v38-balazs15, pmlr-v70-amos17b}) to approximate any convex function $f:\mathbb{R}^d \to \mathbb{R}$ by functions of the form:
\begin{equation}
    \label{max_affine_functions}
    \mathbb{R}^d \ni x \mapsto \underset{i=1:n}{\max} \hspace{0.1cm} \varphi(x; \bm{w_i}) := \underset{i=1:n}{\max} \hspace{0.1cm} \big[\langle w_i, x\rangle + w_i^0  \big], \quad \bm{w_i} := ( w_i, w_i^0) \in \mathbb{R}^{d+1}.
\end{equation}
This observation, the building block of our method, is thoroughly discussed in the next section.

\subsection{Architecture and training}
At first, disregarding the maximum function, function \eqref{max_affine_functions} can be think as a \emph{Linear Neural Network} (see \cite{pmlr-v70-amos17b, JMLR:v25:23-0493}) with one layer. Adding the maximum function makes the function \eqref{max_affine_functions} boils down to a classic Feedforward Neural Network (FNN) with one hidden layer and the maximum function serving as the output activation function. That is, we consider a FNN with one layer of the form:
\begin{equation}
\label{icnn_one_layer}
    \mathbb{R}^d \ni x \mapsto f_{n, 1}(x; \mathcal{W}) := \phi \circ a^{\mathcal{W}}(x) \in \mathbb{R}, \quad 
\end{equation}
where for any $x \in \mathbb{R}^d$:
\begin{equation}
    a^{\mathcal{W}}(x) := Wx + W^0, \quad \mathcal{W} := (W, W^0) := (\bm{w_i})_{i \in \{1,\ldots,n\}} \in \mathbb{M}_{n, d} \times \mathbb{R}^n \quad \text{with} \quad  \bm{w_i} := ( w_i, w_i^0) \in \mathbb{R}^{d+1},
\end{equation}
and the function $\phi$ in \eqref{icnn_one_layer} is the maximum function:
\begin{equation}
    \label{max_function}
    \phi_n^{\max} : \mathbb{R}^n \ni (y_1, \ldots, y_n) \mapsto \max(y_1, \ldots, y_n).
\end{equation}
For optimization purposes, we can consider a regularization of the maximum function based on the classic \emph{LogSumExp} function. That is, instead of function $\phi_n^{\max}$, one can consider the following alternative:
\begin{equation}
    \label{log_sum_exp_function}
    \phi_n^{\lambda}  : \mathbb{R}^n \ni (y_1, \ldots, y_n) \mapsto \frac{1}{\lambda} \log\Bigg(\sum_{i= 1}^n e^{\lambda y_i}\Bigg), \quad \lambda > 0
\end{equation}
which is legitimated by the fact that for any $(y_1, \ldots, y_n) \in \mathbb{R}^n $, one has:
\begin{equation}
    \label{ineq_lse}
    \max(y_1, \ldots, y_n) \le \phi_n^{\lambda}(y_1, \ldots, y_n)   \le \max(y_1, \ldots, y_n) + \frac{\log(n)}{\lambda}
\end{equation}
so that for all $n \in \mathbb{N}^{*}$, $\phi_n^{\lambda}(y_1, \ldots, y_n) \xrightarrow[\lambda \to +\infty]{} \max(y_1, \ldots, y_n)$. Moreover, note that, $\phi_n^{\lambda}$ is convex and non-decreasing so that, disregarding the activation function used, networks of the form \eqref{icnn_one_layer} still produce a convex function.

In practice, we have noticed, through some numerical experiments, that networks of the form \eqref{icnn_one_layer} are not satisfactory because of a slow training time, the loss function getting stuck after a certain number of iterations (see later in Figure \ref{toy_iccn_appr_loss}). This pushed us to consider what we call \emph{Scrambling}. The latter builds upon the fact that composing affine functions remains affine. Thus, one can add several hidden layers on the previous network \eqref{icnn_one_layer} but without activation functions between subsequent layers. That is, we consider a \emph{FNN} of the form:
\begin{equation}
    \label{icnn_function_def}
    \mathbb{R}^d \ni x \mapsto f_{n, L}(x; \bm{\mathcal{W}}) := \phi \circ a_L^{\mathcal{W}_L} \circ a_{L-1}^{\mathcal{W}_{L-1}} \circ \cdots \circ a_1^{\mathcal{W}_1}(x) \in \mathbb{R},
\end{equation}
where,
\begin{itemize}
\item $\bm{\mathcal{W}} := \big(\mathcal{W}_1, \dots, \mathcal{W}_L \big) \in \Theta := \prod_{\ell=1}^L \big(\mathbb{M}_{q_\ell, q_{\ell-1}} \times \mathbb{R}^{q_\ell}\big)$ with $q_\ell$ being integers denoting the number of units per layer. We set $q_0 = d$ and $q_\ell = n$ for $\ell \in \{1,\ldots, L\}$.
\item $L$ denotes the number of layers. For each layer, the affine function $a_\ell^{\mathcal{W}_\ell}:\mathbb{R}^{q_{\ell-1}} \to \mathbb{R}^{q_\ell}$ is defined by $a_\ell^{\mathcal{W}_\ell}(x) = W_\ell \cdot x + W_\ell^0$ with $\mathcal{W}_\ell :=(W_\ell, W_\ell^0) = (\bm{w_{\ell, i}})_{i \in \{1,\ldots,n\}} \in \mathbb{M}_{q_\ell, q_{\ell-1}} \times \mathbb{R}^{q_\ell}$ and $\bm{w_{\ell, i}} = (w_{\ell, i}, w_{\ell, i}^0) \in \mathbb{R}^{q_{\ell-1} + 1}$.
\item $\phi:\mathbb{R}^n \to \mathbb{R}$ is the activation function given either by \eqref{max_function} or \eqref{log_sum_exp_function}.
\end{itemize}

This use of \emph{Scrambling} relates to the \emph{attention mechanism} \cite{NIPS2017_3f5ee243}. Before discussing this connection let us highlight a numerical feature that sheds light on \emph{Scrambling}'s effectiveness. In the representation \eqref{max_affine_functions}, we say that \textit{$k$ (with $1 \le k \le n$) hyperplanes are activated}, if $k$ (out of $n$) hyperplanes achieve the maximum in \eqref{max_affine_functions}. We use equation \eqref{max_affine_functions} here because, with the regularization of the maximum function, it is no longer appropriate to speak of hyperplanes. Our experiments show that with \emph{Scrambling}, multiple hyperplanes are activated, whereas, in the simpler \emph{LinearMax} network, only a few hyperplanes are activated. This illustrates how the weight mixing in \emph{Scrambling} enhances flexibility, enabling all hyperplanes to adjust collectively: adjusting weights in one hyperplane also affects the others due to the interconnected weight products across layers.

\begin{remark}[Connection with a attention mechanism]
The introduction of one or more linear hidden layers (without any non-linear activation function) between other layers might initially seem counter-intuitive. Indeed, from a purely mathematical perspective, the composition of two linear transformations is equivalent to a single linear transformation. As such, the addition of multiple linear layers does not increase the representational power of the network. However, this equivalence does not hold in the context of optimization. When multiple linear layers are used, the system becomes over-parameterized, leading to a scenario where the parameters of interest are derived as a linear combination of auxiliary parameters. This over-parameterization introduces additional degrees of freedom that can have significant implications during training. Specifically, in gradient-based optimization algorithms like gradient descent, the auxiliary parameters are updated iteratively. These updates implicitly influence the effective parameters of the network, allowing for potentially smoother or more efficient convergence dynamics. As a result, while the expressiveness of the model remains unchanged, the optimization landscape may be altered in a way that benefits training. This idea can be compared to the first linear transformation in the attention mechanism, where input vectors are transformed into queries, keys, and values through linear layers. In this case, the linear layers play a critical role in re-parameterizing the input data into different subspaces, each optimized for a specific role in the attention mechanism. Although the linear transformations in attention are mathematically simple, their impact on the optimization process is substantial. By projecting the data into these learned subspaces, the network is able to selectively focus on relevant parts of the input, effectively \q{scrambling} the information in a way that enhances both the model’s learning efficiency and its ability to capture complex relationships. Similarly, in over-parameterized linear networks, the optimization process benefits from the flexibility provided by auxiliary parameters, which can enhance convergence even when the model’s expressive capacity is theoretically unchanged.
\end{remark}

Depending on the activation function used for the output layer, two networks are considered:
\begin{equation}
    \label{icnn_base_max}
    \mathbb{R}^d \ni x \mapsto f_{n, L}^{\max}(x; \bm{\mathcal{W}}) := \phi_n^{\max} \circ a_L^{\mathcal{W}_L} \circ a_{L-1}^{\mathcal{W}_{L-1}} \circ \cdots \circ a_1^{\mathcal{W}_1}(x),
\end{equation}
\begin{equation}
    \label{icnn_lse}
    \mathbb{R}^d \ni x \mapsto f_{n, L}^{\lambda}(x; \bm{\mathcal{W}}) := \phi_n^{\lambda} \circ a_L^{\mathcal{W}_L} \circ a_{L-1}^{\mathcal{W}_{L-1}} \circ \cdots \circ a_1^{\mathcal{W}_1}(x).
\end{equation}
Besides, we consider the following configurations:
\begin{countlist}[label={(\alph*)}]{otherlist2}
  \item \label{iccn_1l_names} A network with no hidden layer i.e. $L = 1$ and the two versions (regular and smooth) of the maximum function yielding respectively networks named \textbf{\emph{LinearMax (LM)}} (regular) and \textbf{\emph{LinearLogSumExp (L2SE)}} (smooth).

  \item \label{iccn_sevL_names} For $L > 1$, we get the \emph{Scrambling} version by considering variations of \emph{LinearMax (LM)} and \emph{LinearLogSumExp (L2SE)} with $L-1$ hidden layers yielding networks named \textbf{\emph{L-ScrambledLinearMax (L-SLM)}} and \textbf{\emph{L-ScrambledLinearLogSumExp (L-SL2SE)}}.
\end{countlist}

We could have labeled the networks defined in \ref{iccn_1l_names} with the generic names outlined in \ref{iccn_sevL_names}, but we chose not to in order to emphasize the following point. Networks \emph{LinearMax (LM)} or \emph{LinearLogSumExp (L2SE)}, where $L = 1$, are given by \eqref{icnn_one_layer}, corresponding to a straightforward application of Proposition \ref{charac_convex_func_max_affine} as explained earlier. However, in the \emph{Scrambling}, networks \emph{L-ScrambledLinearMax (L-SLM)} and \emph{L-ScrambledLinearLogSumExp (L-SL2SE)} extend the preceding straightforward application by using the fact that the composition of linear functions remains linear. This allows for stacking hidden layers without activation functions while still producing a maximum of affine functions as output. An illustration of our networks is shown in Figure \ref{linear_max_net} and Figure \ref{scrambled_net}.

\begin{figure}[ht!]
    \centering
    \includegraphics[width=0.7\columnwidth]{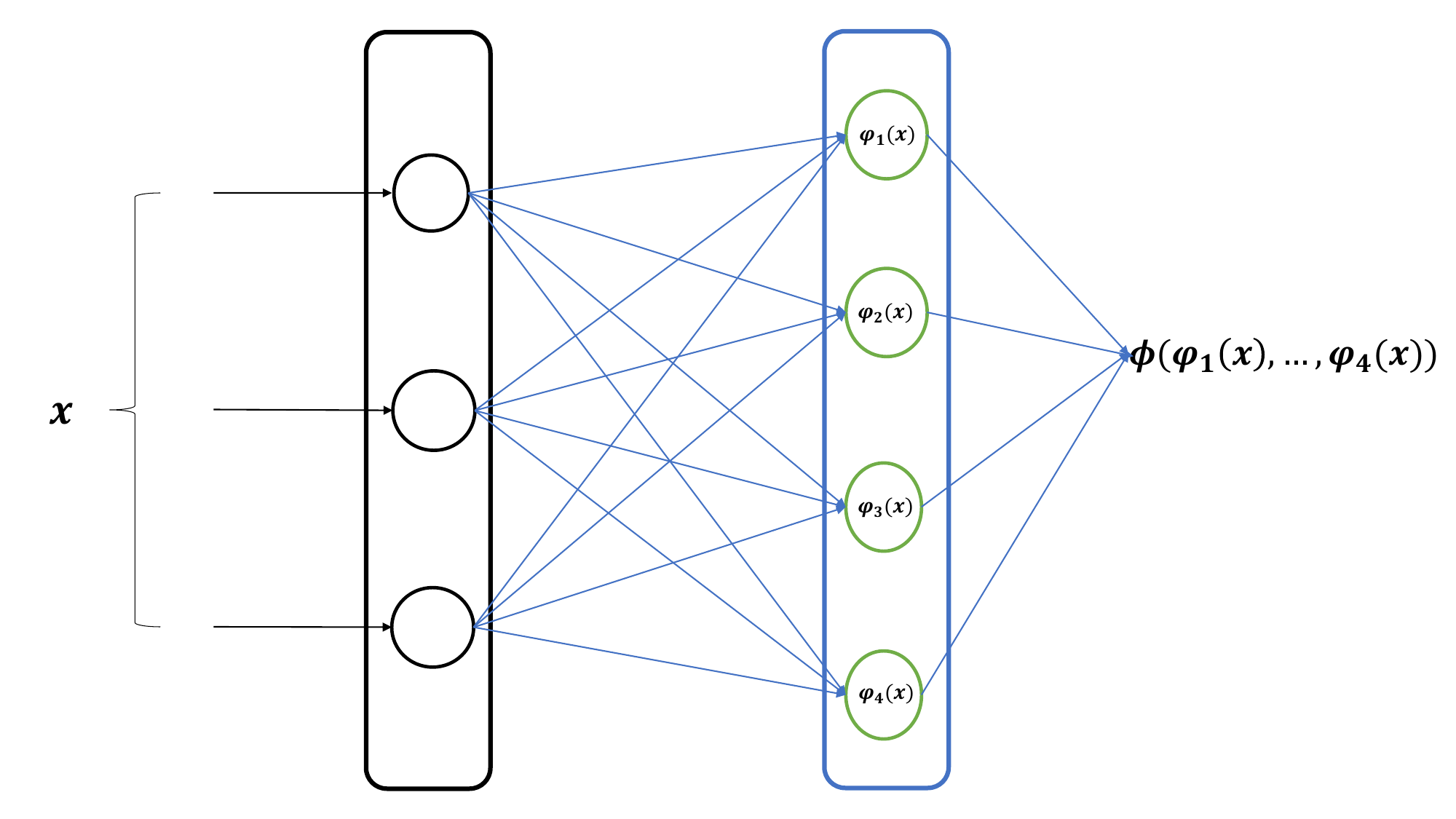}
    \caption{\textit{Illustration of a fully Convex Network architecture: simple Linear network ($L=1$).}}
    \label{linear_max_net}
\end{figure}

\begin{figure}[ht!]
    \centering
    \includegraphics[width=0.7\columnwidth]{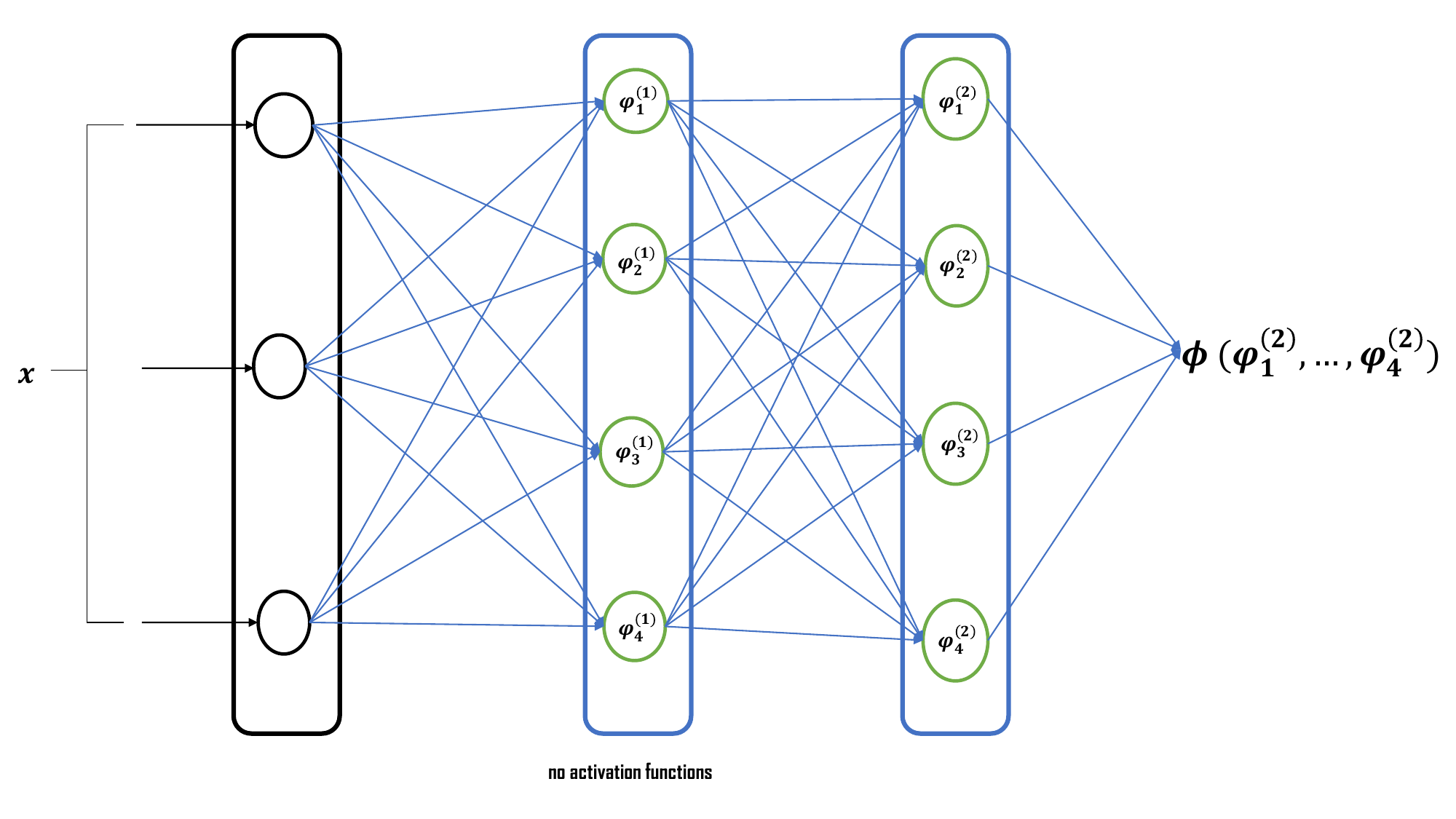}
    \caption{\textit{Illustration of a fully Convex Network architecture: scrambled Linear network ($L=2$).}}
    \label{scrambled_net}
\end{figure}

\begin{remark}[Parameter $\lambda$ for $\phi_n^{\lambda}$]
\label{choix_lambda}
    The parameter $\lambda$ in \eqref{log_sum_exp_function} has to be chosen carefully to insure the convergence of the $\phi_n^{\lambda}$ function toward the maximum function. Additionally, $\lambda$ acts as a scaling parameter which cannot be set disregarding data $(y_1, \ldots, y_n)$ for which we want to compute the maximum. For that reason and to allow for flexibility, we consider $\lambda$ of the form $c \cdot \Tilde{\lambda}$, where $\Tilde{\lambda}$ is a trainable parameter and $c$ is a fixed constant.
\end{remark}

Once the \emph{Convex Network} architecture is defined, it needs to be trained to solve a learning task which in our case is the approximation of a convex function $f:\mathbb{R}^d \to \mathbb{R}$. To achieve this, we build a network $f_n^{\max}(x;\bm{\mathcal{W}})$ (or $f_n^{\lambda}(x;\bm{\mathcal{W}})$) of the form \eqref{icnn_base_max} (or \eqref{icnn_lse}), and determine its weights by minimizing the corresponding loss function:
\begin{equation}
\label{loss_icnn}
     \underset{\bm{\mathcal{W}}}{\inf} \hspace{0.1cm} \Bigg[\int_{\mathbb{R}^d}^{} \big|f(x) - f_{n, L}^{\max}(x;\bm{\mathcal{W}}) \big|^2 \, \mu(\mathrm{d}x) \Bigg]^{1/2} \quad \text{or} \quad \underset{\bm{\mathcal{W}}}{\inf} \hspace{0.1cm} \Bigg[\int_{\mathbb{R}^d}^{} \big|f(x) - f_{n,L}^{\lambda}(x;\bm{\mathcal{W}})\big|^2 \, \mu(\mathrm{d}x) \Bigg]^{1/2},
\end{equation}

\noindent
where $\mu$ is the distribution of input data $x$. Besides, the minimization of either loss functions can be performed by a Stochastic Gradient Descent method. The approximation capacity of such a neural network trained to minimize loss functions defined in \eqref{loss_icnn} is studied in the following section.

\subsection{Convex functions approximation}
Build a \emph{Convex Network} of the form \eqref{icnn_base_max} or \eqref{icnn_lse} which minimizes either loss functions in \eqref{loss_icnn} yields an approximation error. This error may be analyzed in different ways. In this paper, we present two results, providing error bounds for this approximation on compact sets and in $\mathbb{L}_{\mathbb{R}^d}^p$ spaces. We start with the analysis of the approximation error on compact sets.

Let $A < B$ be real numbers and consider an $\ell^{\infty}$-ball $K = [A, B]^d \subset \mathbb{R}^d$. Then, for any convex function $f:\mathbb{R}^d \to \mathbb{R}$, define the following two approximation errors on the compact set $K$ by:
\begin{equation}
    \label{approx_error_compact_set_base}
    \mathcal{E}_K(f, f_n^{\max}) := \underset{\mathcal{W} = (\bm{w_i})_{i=1:n} = (w_i, w_i^0)_{i=1:n}}{\min} \hspace{0.1cm} \Bigg[ \underset{x \in K}{\sup} \hspace{0.1cm}\Big|f(x) - f_n^{\max}(x; \mathcal{W})\Big| \Bigg],
\end{equation}
\begin{equation}
    \label{approx_error_compact_set_lse}
    \mathcal{E}_K(f, f_n^{\lambda}) := \underset{\mathcal{W} = (\bm{w_i})_{i=1:n} = (w_i, w_i^0)_{i=1:n}}{\min} \hspace{0.1cm} \Bigg[ \underset{x \in K}{\sup} \hspace{0.1cm}\Big|f(x) - f_n^{\lambda}(x; \mathcal{W})\Big| \Bigg].
\end{equation}
In the preceding definition of the errors, networks $f_n^{\max}, f_n^{\lambda}$ correspond to the simple version of our \emph{Convex Network} (with $L = 1$, see \eqref{icnn_one_layer}). All the results in this section can be extended to the \emph{scrambling} version in a straightforward manner since, from a mathematical standpoint, composing affine functions still yields an affine function. Thus, in this section, we drop the index $L$ (in \eqref{icnn_base_max} or in \eqref{icnn_lse}) representing the number of layers.

Then, we have the following convergence result.

\begin{theorem}[Uniform approximation on compact set]
\label{thm_error_approx_icnn_sur_compact}
    If $f : K \to \mathbb{R}$ is convex, continuously differentiable on $K$, with a bounded gradient $\nabla f$, then one has the following two bounds:
    \begin{equation}
        \label{error_on_compact}
        \mathcal{E}_K(f, f_n^{\max}) = \mathcal{O}\big(n^{-2/d}\big) \quad \text{as} \quad n \to +\infty.
    \end{equation}
        \begin{equation}
        \label{error_on_compact_lse}
        \mathcal{E}_K(f, f_n^{\lambda}) = \mathcal{O}\Big(n^{-2/d} + \frac{\log(n)}{\lambda}\Big) \quad \text{as} \quad n \to +\infty.
    \end{equation}
\end{theorem}

\begin{remark}
    In light of the preceding Theorem, in the regularized version (see \eqref{log_sum_exp_function}) of the maximum function, one can choose the parameter $\lambda$ wisely. For example in \eqref{error_on_compact_lse}, one may set:
    $$\lambda := n^{\frac{2}{d} + \frac{1}{2}} \implies \mathcal{E}_K(f, f_n^{\lambda}) = \mathcal{O}\big(n^{-2/d}\big).$$
    The same holds true for Theorem \ref{thm_error_approx_icnn_sur_L2}, where $\lambda$ can also be chosen wisely.
\end{remark}

\begin{proof}[Proof of Theorem \ref{thm_error_approx_icnn_sur_compact}]
    From the definition of network $f_n^{\max}$, we get:
    $$\mathcal{E}_K(f, f_n^{\max}) := \underset{(\bm{w_i})_{i=1:n}}{\min} \hspace{0.1cm} \Bigg[ \underset{x \in K}{\sup} \hspace{0.1cm}\Big|f(x) - \underset{i=1:n}{\max} \hspace{0.1cm} \varphi(x; \bm{w_i})\Big| \Bigg],$$

    \noindent
    where $\varphi(x; \bm{w_i})$ is of the form \eqref{max_affine_functions}. Then, claim \eqref{error_on_compact} follows from a proof, based on covering numbers, which may be found in \cite{pmlr-v38-balazs15}.

    \vspace{0.2cm}
    For the second error bound, we get by triangle inequality that:
    $$\mathcal{E}_K(f, f_n^{\lambda}) \le \mathcal{E}_K(f, f_n^{\max}) + \underset{\mathcal{W}}{\min} \hspace{0.1cm} \Bigg[ \underset{x \in K}{\sup} \hspace{0.1cm}\Big|f_n^{\max}(x; \mathcal{W}) - f_n^{\lambda}(x; \mathcal{W})\Big| \Bigg] \le \mathcal{E}_K(f, f_n^{\max}) + \frac{\log(n)}{\lambda},$$
    \noindent
    where the last inequality comes from the definition of networks $f_n^{\max}, f_n^{\lambda}$ (see \eqref{icnn_base_max} and \eqref{icnn_lse}) combined with the inequality \eqref{ineq_lse}. Then, combining the preceding inequality with \eqref{error_on_compact} gives the desired result.
\end{proof}

We now study the $\mathbb{L}_{\mathbb{R}^d}^r(\mu)$-error, where $\mu$ is the distribution of input data $x$. That is, we consider the following errors:
\begin{equation}
    \label{error_approx_mu}
    \mathcal{E}_{r, \mu}(f, f_n^{\max}) := \underset{\mathcal{W} = (\bm{w_i})_{i=1:n} = (w_i, w_i^0)_{i=1:n}}{\min} \hspace{0.1cm} \Bigg[\int_{\mathbb{R}^d}^{} \Big|f(x) - f_n^{\max}(x; \mathcal{W}) \Big|^r \, \mu(\mathrm{d}x) \Bigg]^{1/r},
\end{equation}
\begin{equation}
    \label{error_approx_mu_lse}
    \mathcal{E}_{r, \mu}(f, f_n^{\lambda}) := \underset{\mathcal{W} = (\bm{w_i})_{i=1:n} = (w_i, w_i^0)_{i=1:n}}{\min} \hspace{0.1cm} \Bigg[\int_{\mathbb{R}^d}^{} \Big|f(x) - f_n^{\lambda}(x; \mathcal{W}) \Big|^r \, \mu(\mathrm{d}x) \Bigg]^{1/r}.
\end{equation}
Here again, we disregard the number of layers $L$. We then have the second convergence result.

\begin{theorem}[$\mathbb{L}_{\mathbb{R}^d}^p(\mu)$-approximation on $\mathbb{R}^d$]
\label{thm_error_approx_icnn_sur_L2}
\begin{countlist}[label={(\alph*)}]{otherlist1}
    \item \label{bound_error_lp_by_quantif} Let $f:\mathbb{R}^d \to \mathbb{R}$ be a convex differentiable function with $\alpha$-Holder gradient for $\alpha > 0$. Then, for every $r \in \big(0, \frac{2}{\alpha + 1}\big]$ and every distribution $\mu$ on $(\mathbb{R}^d,{\cal B}or (\mathbb{R}^d))$ having a finite second moment, one has:
    \begin{equation}
        \label{error_approx_fnc_error_quantif}
        {\cal E}_{r,\mu}(f, f_n^{\max})\le [\nabla f]_{\alpha, H} \cdot e_{2,n}(\mu)^{\alpha + 1},
    \end{equation}
        \begin{equation}
        \label{error_approx_fnc_error_quantif_lse}
        {\cal E}_{r,\mu}(f, f_n^{\lambda})\le [\nabla f]_{\alpha, H} \cdot e_{2,n}(\mu)^{\alpha + 1} + \frac{\log(n)}{\lambda},
    \end{equation}

where $[\nabla f]_{\alpha, H}$ denotes the Holder constant of the gradient of $f$ and for $p > 0$:
  \begin{equation}
      \label{quantif_error}
      e_{p, n}(\mu) := \underset{(x_1,\ldots, x_n) \in (\mathbb{R}^d)^n}{\inf} \hspace{0.1cm} \Big\|\underset{i=1:n}{\min} \hspace{0.1cm} \big|x - x_i\big|  \Big\|_{\mathbb{L}^p(\mu(\mathrm{d}x))}.
  \end{equation}

  \item \label{finite_moment_hyp_error_Lp} Moreover, if $\mu$ has a finite moment of order $2+\delta$ for some $\delta>0$ then 
  \begin{equation}
      \label{quantif_error_2}
      {\cal E}_{r,\mu}(f,f_n^{\max}) \le C_{d,\delta} \cdot [\nabla f]_{\alpha, H} \cdot \sigma^{\alpha+1}_{2+\delta}(\mu) n^{-\frac{\alpha+1}{d}} = O(n^{-(\alpha+1)/d}),
  \end{equation}
  \begin{equation}
      \label{quantif_error_2_lse}
      {\cal E}_{r,\mu}(f,f_n^{\lambda}) = O\Big(n^{-(\alpha+1)/d} + \frac{\log(n)}{\lambda}\Big),
  \end{equation}
where the positive constant $C_{d,\delta}$ does not depend on $\mu$ and for $p > 0$, $\sigma_p(\mu)$ is defined by:
$$\sigma_p(\mu) := \inf_{a\in \mathbb{R}^d} \|X-x\|_p.$$

\noindent
Note that, \eqref{quantif_error_2} and \eqref{quantif_error_2_lse} still hold true if $\mu$ is radial for some norm $N_0$ outside a compact K (i.e. $\mu= h\cdot \lambda_d$ with $h(\xi) = g(N_0(\xi-a))$ for every $\xi \notin K$, where $a\in \mathbb{R}^d$ with $h$ assumed bounded). In the latter case, one can set $\delta =0$ (quasi-radial setting).

  \item \label{sharp_rate_error} If the distribution $\mu = h \cdot \lambda_d$ has a finite moment of order $\frac{sd}{2+d-s} + \delta$ for some $s \in \big(\frac{2}{1+\alpha}, \frac{2+d}{1+\alpha}\big)$ and $\delta > 0$, and if assumption in \ref{bound_error_lp_by_quantif} are in force, then \eqref{quantif_error_2} and \eqref{quantif_error_2_lse} still hold.

  \item \label{more_sharp_rate_error} {(\textit{Asymptotic bounds})} If the distribution $\mu = h \cdot \lambda_d$ has a finite moment of order $r+\delta$ for some $\delta > 0$ and $f$ is $\mathcal{C}^2$ with a uniformly continuous bounded Hessian $\nabla^2 f$, then one has:
  $$\underset{n}{\lim} \hspace{0.1cm} n^{1/d} \cdot \mathcal{E}_{r, \mu}(f, f_n^{\max}) \le  \mathcal{J}_{r, d} \quad \text{and} \quad \underset{n}{\lim} \hspace{0.1cm} n^{1/d} \cdot \mathcal{E}_{r, \mu}(f, f_n^{\lambda}) \le \mathcal{J}_{r, d}  + n^{1/d}\frac{\log(n)}{\lambda},$$
  where,
  $$\mathcal{J}_{r, d} = \frac{1}{\sqrt{2}}\Bigg(\int_{\mathbb{R}^d}^{} \det\big(\nabla^2f(x)\big)^{\frac{r}{2(d+r)}} \cdot h^{\frac{d}{d+r}}(x) \, \lambda_d(\mathrm{d}x)  \Bigg)^{\frac{1}{r} + \frac{1}{d}}.$$
\end{countlist}    
\end{theorem}
An important part of the proof of the preceding Theorem relies on how one may bound the error term $e_{p, n}(\mu)$ in \eqref{quantif_error}. It turns out that this error term is the $\mathbb{L}^p$-quantization error with respect to the input distribution $\mu$. Let us recall a short background behind \emph{optimal quantization theory} \cite{LuschgyQuantif, bookGillesQuantifMarginale, Pagès2018Quantif} which aims at approximating continuous random vectors by discrete ones.

\begin{definition}
    \begin{countlist}[label={(\alph*)}]{otherlist5}
        \item Let $\alpha_N = (x_1,\ldots,x_N) \subset \mathbb{R}^d$ be a subset of size $N$, called $N$-quantizer. A Borel partition $(C_i^N(\alpha_N))_{1 \le i \le N} := (C_i^N)_{1 \le i \le N}$ of $\mathbb{R}^d$ is a \emph{Voronoi partition} of $\mathbb{R}^d$ induced by the $N$-quantizer $\alpha_N$ if, for every $i \in \{1,\ldots, N\}$,
        $$C_i^N \subseteq \Big\{\xi \in \mathbb{R}^d, \quad |\xi-x_i|  \le \underset{j \neq i}{\min} \hspace{0.1cm} |\xi-x_j|  \Big\}.$$
        Points $x_i$ are called \emph{Centroids} of the Borel sets $C_i^N$ which in turn are called \emph{Voronoi cells} of the partition induced by the $N$-quantizer $\alpha_N$.
        
        \item Let $X \in \mathbb{L}_{\mathbb{R}^d}^2(\Omega, \mathcal{A}, \mathbb{P})$ for some probability space $(\Omega, \mathcal{A}, \mathbb{P})$. The Voronoi quantization of the random vector $X$ induced by $(C_i^N)_{1 \le i \le N}$ is defined by the projection of $X$ onto the Voronoi partition $(C_i^N)_{1 \le i \le N}$:
        $$\hat{X}^{\alpha_N} := \sum_{i = 1}^{N} x_i \mathbf{1}_{C_i^N(\alpha_N)}(X).$$

        \item The subset $\alpha_N = (x_1,\ldots,x_N)$ is an \emph{$\mathbb{L}^p$-$N$-quantizer} for the random vector $X$ if it minimizes the $\mathbb{L}^p_{\mathbb{R}^d}$-quantization error:
        \begin{equation}
            \label{min_lp_quantif_error}
            \underset{\bm{y}_N := (y_1,\ldots,y_N) \in \big(\mathbb{R}^d\big)^N}{\inf} \hspace{0.2cm} \big\| X-\hat{X}^{\bm{y}_N} \big\|_{\mathbb{L}^p_{\mathbb{R}^d}(\mathbb{P})} = \underset{\bm{y}_N := (y_1,\ldots,y_N) \in \big(\mathbb{R}^d\big)^N}{\inf} \hspace{0.2cm} \Big\|\underset{1 \le i \le N}{\min}\big|X - y_i \big|  \Big\|_{\mathbb{L}^p_{\mathbb{R}^d}(\mathbb{P})}.
        \end{equation}
    \end{countlist}
\end{definition}

It is classic background (see Theorem 5.1 in \cite{Pagès2018Quantif}) that the minimal $\mathbb{L}^p_{\mathbb{R}^d}$-quantization error (see \eqref{quantif_error} or \eqref{min_lp_quantif_error}), attained for an $\mathbb{L}^p$-$N$-quantizer, converges to $0$ as $N \to + \infty$. Zador's theorem (reproduced in Theorem \ref{Zador_thm} for the reader convenience), gives the rate of this convergence. We now have basic blocks to prove Theorem \ref{thm_error_approx_icnn_sur_L2}.

\begin{proof}[Proof of Theorem \ref{thm_error_approx_icnn_sur_L2}]
We only prove results on error ${\cal E}_{r,\mu}(f, f_n^{\max})$ since all the results established for ${\cal E}_{r,\mu}(f, f_n^{\max})$ systematically yields the results for ${\cal E}_{r,\mu}(f, f_n^{\lambda})$ by applying triangle inequality and inequalities \eqref{ineq_lse}.

\begin{countlist}[label={(\alph*)}]{otherlist4}
  \item First, notice that the function:
  $$\varphi_i(x) := \langle  \nabla f(x_i), x -x_i\rangle  + f(x_i)$$
  may be rewritten as follows:
  $$\varphi_i(x) = \langle w_i, x \rangle + w_i^0 \quad \text{with} \quad w_i = \nabla f(x_i) \in \mathbb{R}^d, w_i^0 = f(x_i) - \langle \nabla f(x_i), x_i\rangle \in\mathbb{R}.$$
  Then, it follows from the convexity of $f$ and then Cauchy-Schwartz inequality that:
  \begin{align*}
      0 \le f(x) - \underset{i=1:n}{\max} \hspace{0.1cm} \varphi_i(x) &= \underset{i=1:n}{\min} \hspace{0.1cm} \Big[f(x) - f(x_i) - \langle \nabla f(x_i), x-x_i  \rangle \Big]\\
      &\le \underset{i=1:n}{\min} \hspace{0.1cm} \Big[\langle \nabla f(x), x- x_i\rangle - \langle \nabla f(x_i), x-x_i  \rangle \Big]\\
      &= \underset{i=1:n}{\min} \hspace{0.1cm} \Big[\langle \nabla f(x) - \nabla f(x_i), x- x_i\rangle\Big]\\
      &\le \big[\nabla f \big]_{\alpha, H} \cdot \underset{i=1:n}{\min} \hspace{0.1cm} |x-x_i|^{\alpha + 1}.
  \end{align*}

  \noindent
  So that for every $r \in \big(0, \frac{2}{\alpha + 1}\big]$, one has:
  \begin{align*}
      \Big[\int_{\mathbb{R}^d}^{} \Big| f(x) - \underset{i=1:n}{\max} \hspace{0.1cm} \varphi_i(x) \Big|^r \, \mu(\mathrm{d}x) \Big]^{\frac{1}{r}} &\le \big[\nabla f \big]_{\alpha, H} \cdot  \Big\|\underset{i=1:n}{\min} \hspace{0.1cm} |x-x_i| \Big\|^{\alpha + 1}_{\mathbb{L}^{(\alpha+1) r}(\mu(\mathrm{d}x))}\\
      &\le \big[\nabla f \big]_{\alpha, H} \cdot  \Big\|\underset{i=1:n}{\min} \hspace{0.1cm} |x-x_i| \Big\|^{\alpha + 1}_{\mathbb{L}^{2}(\mu(\mathrm{d}x))}.
  \end{align*}

  \noindent
  This holds true for any $n$-tuple $(x_1,\ldots,x_n)$ so that:
  $$\Big[\int_{\mathbb{R}^d}^{} \Big| f(x) - \underset{i=1:n}{\max} \hspace{0.1cm} \varphi_i(x) \Big|^r \, \mu(\mathrm{d}x) \Big]^{\frac{1}{r}} \le [\nabla f]_{\alpha, H} \cdot e_{2,n}(\mu)^{\alpha + 1}.$$

  \noindent
  But, as pointed out at the beginning of the proof, $\underset{i=1:n}{\max} \hspace{0.1cm} \varphi_i(x)$ can be rewritten as $\underset{i=1:n}{\max} \hspace{0.1cm} \big[\langle w_i, x\rangle + w_i^0  \big]$ with $\bm{w_i} = (w_i, w_i^0) \in \mathbb{R}^{d+1}$. Therefore, by taking the infimum over $(\bm{w_i})_{i \in\{1,\ldots,n\}}$, we conclude that:
  $${\cal E}_{r,\mu}(f, f_n^{\max})\le [\nabla f]_{\alpha, H} \cdot e_{2,n}(\mu)^{\alpha + 1}.$$

  \item The first part of this claim follows from the non-asymptotic Zador’s Theorem \ref{Zador_thm} (from optimal quantization theory) combined with inequalities \eqref{error_approx_fnc_error_quantif} and \eqref{error_approx_fnc_error_quantif_lse}. For the second part, we refer the reader to Theorem 2.1.17 in \cite{bookGillesQuantifMarginale}.

\item This follows from the proof of claim \ref{bound_error_lp_by_quantif} and a straightforward application of Theorem 4.3 in \cite{PAGES2018847}.

\item This follows from Zador's extension to Bregman divergence (see \cite{NIPS2016_c4851e8e, thesisGuillaume}).
\end{countlist} 
\end{proof}

We now have theoretical evidences of the performance of \emph{Convex Networks} of the form \eqref{icnn_base_max} or \eqref{icnn_lse} trained to minimize either loss functions in \eqref{loss_icnn}. Now, our focus is on the practical side, dealing with option pricing, especially we will consider options with convex payoffs.

\section{Options pricing through convexity}
\label{section_2}
In financial markets, an option is a financial derivative product giving the right (not the obligation) to its holder to buy (referred to as \emph{Call}) or sell (referred to as \emph{put}) a financial asset called the \emph{underlying asset}. We consider options on multi-assets and assume their dynamics is driven by the \emph{Black \& Scholes} model, where underlying assets are $d$ risky assets modelled by the following diffusion (in the risk-neutral world) for $i \in \{1,\ldots,d\}$ and $t \ge 0$:
\begin{equation}
    \label{bs_model}
    \frac{dS_t^i}{S_t^i} = (r-\delta_i) dt + \sigma_i dW_t^i,\quad S_0^i = s_0^i \in \mathbb{R}_{+}^{*},
\end{equation}

\noindent
where $r > 0$ is the risk less rate, $\sigma_i > 0$ and $\delta_i > 0$ are volatilities and dividend rates of each asset, and $\big(W_t^1,\ldots,W_t^d\big)$ is a $d$-dimensional Brownian motion with instantaneous correlation $\rho_{i, j}$ for $i,j \in \{1,\ldots, d\}$. Note that, in this model, the risky assets prices are given by:
\begin{equation}
    \label{asset_price_bs_model}
    S_t^i = s_0^i e^{\big(r - \delta_i - \frac{\sigma_i^2}{2}\big)t + \sigma_i W_t^i}, \quad i \in \{1,\ldots, d\}.
\end{equation}

In this paper, our main focus is the pricing of options with a convex payoff with respect to the initial asset prices $(s_0^1,\ldots,s_0^d) \in \big(\mathbb{R}_{+}^{*}\big)^d$. In what follows, we consider three different types (with increasing complexity) of those options.

\subsection{Vanilla options: Basket option}
We first consider \emph{vanilla options} which are options which value depends only on the underlying asset value at certain terminal date (also called \emph{expiry}) $T$. Given a payoff function $g : \mathbb{R}^d \to \mathbb{R}$ of such an option, their price under the risk neutral probability is given by:
\begin{equation}
    \label{price_option_vanille}
    v_0 := v(s_0^1,\ldots,s_0^d) := \mathbb{E}\Big[e^{-rT}g\big(S_T^1,\ldots,S_T^d\big)\Big].
\end{equation}

In this paper, we consider a \emph{basket call option} which gives the right to buy, at the expiry $T$ and at a certain \emph{strike price} $K > 0$, a portfolio which is a linear (convex) combination of underlying assets modelled by \eqref{asset_price_bs_model} where we will assume that $\delta_i = 0$ (this assumption does not impact the methodology and is only in force in the case of Basket options). That is:
\begin{equation}
    \label{payoff_basket_call}
    g\big(S_T^1,\ldots,S_T^d\big) := \Bigg(\sum_{i = 1}^d \alpha_i S_T^i - K\Bigg)_{+}, \quad \alpha_i > 0, \quad \sum_{i = 1}^d \alpha_i = 1.
\end{equation}

\begin{Proposition}
    Considering the payoff function of a basket call option given by \eqref{payoff_basket_call}, the price of such option: $\big(\mathbb{R}_{+}^{*}\big)^d \ni (s_0^1,\ldots,s_0^d) \mapsto v(s_0^1,\ldots,s_0^d)$ is a convex function.
\end{Proposition}

\begin{proof}
    The proof is straightforward noticing that, in model \eqref{asset_price_bs_model}, almost surely, $\big(\mathbb{R}_{+}^{*}\big)^d \ni (s_0^1,\ldots,s_0^d) \mapsto g\big(S_T^1,\ldots,S_T^d\big)$ is a convex function.
\end{proof}

The preceding proposition suggests that the function $\big(\mathbb{R}_{+}^{*}\big)^d \ni (s_0^1,\ldots,s_0^d) \mapsto v(s_0^1,\ldots,s_0^d)$ may be approximated by our \emph{Convex Network}. In light of the general training procedure of our \emph{Convex Network} described above, we need to train a network $f_{n, L}(\cdot; \bm{\mathcal{W}})$ (of the form \eqref{icnn_base_max} or \eqref{icnn_lse}) in a way that it minimizes the following loss function:
\begin{equation}
    \label{loss_icnn_approx_vanilla}
    \underset{\bm{\mathcal{W}}}{\inf} \hspace{0.1cm} \Big\| v(s_0^1,\ldots,s_0^d) - f_{n, L}(s_0^1,\ldots,s_0^d; \bm{\mathcal{W}}) \Big\|_{\mathbb{L}_{\mathbb{R}^d}^2(\mu)}, \quad \text{where} \quad (s_0^1,\ldots,s_0^d) \sim \mu,
\end{equation}

\noindent
where the choice of the input distribution $\mu$ is discussed latter.

Solving the preceding optimization problem requires to compute the basket option price $v(s_0^1,\ldots,s_0^d)$ given $\big(\mathbb{R}_{+}^{*}\big)^d \ni (s_0^1,\ldots,s_0^d)$. However, we do not have closed forms for the expectation in \eqref{price_option_vanille} with the payoff function \eqref{payoff_basket_call}. Thus, we replace the preceding expectation by its Monte Carlo counterpart and in order to improve the Monte Carlo procedure, we use a control variate. That is, $v(s_0^1,\ldots,s_0^d)$ in \eqref{loss_icnn_approx_vanilla} is practically replaced with the controlled estimator (see \cite{DINGEC2013421, Pagès2018}):
\begin{equation}
    \label{estimator_with_cv_basket_call}
    \hat{v}_M(s_0^1,\ldots,s_0^d) = \frac{e^{-rT}}{M} \sum_{m = 1}^{M} \Bigg[\sum_{i = 1}^d \Big(\alpha_i S_T^{i, (m)} - K\Big)_{+} - \Big(e^{\sum_{i=1}^d \alpha_i \log \big( S_T^{i, (m)} \big)} - K \Big)_{+}  \Bigg] + \Pi_0,
\end{equation}

\noindent
where given a distribution $\mu$ and $i \in \{1,\ldots,d\}$, $\big(S_T^{1, (m)},\ldots, S_T^{d, (m)}\big)$ are \emph{i.i.d.} copies of $\big(S_T^{1},\ldots, S_T^{d}\big)$. $\Pi_0$ is given by:
$$\Pi_0 := e^{-rT}\mathbb{E}\Bigg[ \Big(e^{\sum_{i=1}^d \alpha_i \log \big( S_T^{i} \big)} - K \Big)_{+} \Bigg] = e^{a + \frac{1}{2}\sigma^2} \Phi(\kappa + \sigma) - K \Phi(\kappa),$$

\noindent
where $\Phi$ is the cumulative distribution function of the standard Normal distribution. Considering the notations $\alpha := (\alpha_1,\ldots,\alpha_d)^{*}, \Sigma := \big[\sigma_i \rho_{i,j}\sigma_j  \big]_{1 \le i,j \le d}$, we also have:
$$\sigma^2 = T \alpha^{*}\Sigma \alpha, \quad a = \sum_{i = 1}^d \alpha_i \Big(\log(s_0^i) + \big(r-\frac{\sigma_i^2}{2}\big)T \Big), \quad \kappa = \frac{a - \log(K)}{\sigma}.$$

\noindent
Recall that, we assumed $\delta_i = 0$ in \eqref{bs_model}. Besides, note that $M$ has to be chosen sufficiently large in order to train the \emph{Convex Network} on, as accurate as possible, data. Now, let us discuss the choice of the training distribution.

\subsubsection*{Practitioner's corner: Choice of $\mu$}
Let $(x_0^1,\ldots,x_0^d) \in (\mathbb{R}_{+}^{*})^d$ and consider $R > 0$ sufficiently large. Our aim is to approximate the convex function $\mathcal{B}\big( (x_0^1,\ldots,x_0^d), R_d\big) \ni (y_0^1,\ldots,y_0^d) \mapsto v(y_0^1,\ldots,y_0^d)$, where $R_d = \sqrt{d} R$. On top of that, we want to prove that, using our \emph{Convex Network} which is convex by nature, allows to accurately reconstitute the whole price curve $\mathcal{B}\big( (x_0^1,\ldots,x_0^d), R_d\big) \ni (y_0^1,\ldots,y_0^d) \mapsto v(y_0^1,\ldots,y_0^d)$ by just learning on few points in the ball $\mathcal{B}\big( (x_0^1,\ldots,x_0^d), R_d\big)$. To this end, we set the initial distribution $\mu$ as follows:
\begin{equation}
    \label{n_strat_s0_input}
    (s_0^1,\ldots,s_0^d) \underset{i.i.d.}{\sim}(x_0^1,\ldots,x_0^d) - R  + 2R \cdot \big(U_1, \ldots, U_d\big) \in \prod_{i=1}^d \big[x_0^i - R, x_0^i + R \big],
\end{equation}
where $(U_i)_{1 \le i \le d}$ is an \emph{i.i.d.} sequence of random variables following the uniform distribution $\mathcal{U}\big([0,1]\big)$. The pseudo code of our training procedure is described in Algorithm \ref{algo_icnn_basket_option}.

\begin{algorithm}
    \SetKwFunction{isOddNumber}{isOddNumber}
    \SetKwInOut{KwIn}{Input}
    \SetKwInOut{KwOut}{Output}
    \SetKwRepeat{Do}{do}{while}%

    \KwIn{
        \begin{itemize}[noitemsep]
            \item Model parameters: $d, r, \sigma_i, T, \rho_{i, j}$.
            \item Other settings: $K, \alpha_i, M, \mu, L, n, R$.
            \item Batch size: $B$, number of iterations: $I$.
        \end{itemize}
    }

    \KwOut{$\mathcal{B}\big( (x_0^1,\ldots,x_0^d), R_d\big) \ni (y_0^1,\ldots,y_0^d) \mapsto v(y_0^1,\ldots,y_0^d)$.}

    \vspace{0.2cm}
    \tcc{Initialize weights $\bm{\mathcal{W}}$.}
    
    \vspace{0.2cm}
    \For{$i\gets1$ \KwTo $I$}{
        Generate $B$ \emph{i.i.d.} input data $(s_0^{1, (b)}, \ldots, s_0^{d, (b)}) \sim \mu$ for $b \in \{1,\ldots, B\}$ using \eqref{n_strat_s0_input}.

        \vspace{0.1cm}
        For each input data $b \in \{1,\ldots, B\}$, simulate an \emph{i.i.d.} $M$-sample $\big(S_T^{1, (b,m)},\ldots, S_T^{d, (b,m)}\big)_{m \in \{1,\ldots, M\}}$ of the random vector $\big(S_T^{1},\ldots, S_T^{d}\big)$ starting at $\big(s_0^{1, (b)}, \ldots, s_0^{d, (b)}\big)$. Then, compute $\hat{v}_M(s_0^{1, (b)}, \ldots, s_0^{d, (b)})$ using \eqref{estimator_with_cv_basket_call}.

        \vspace{0.1cm}
        Update weights $\bm{\mathcal{W}}$ by using Stochastic Gradient Descent to solve:
        $$\underset{\bm{\mathcal{W}}}{\inf} \hspace{0.1cm} \frac{1}{B} \sum_{b = 1}^B \Big(\hat{v}_M\big(s_0^{1, (b)}, \ldots, s_0^{d, (b)}\big) - f_{n, L}\big(s_0^{1, (b)}, \ldots, s_0^{d, (b)}; \bm{\mathcal{W}}\big) \Big)^2.$$
    }

    \KwRet{$f_{n,L}(y_0^1,\ldots,y_0^d; \bm{\mathcal{W}})$ for all $(y_0^1,\ldots,y_0^d) \in \mathcal{B}\big((x_0^1,\ldots,x_0^d), R_d\big)$.}
    \caption{\emph{Convex Network} training: Pricing of a Basket call option.}
\label{algo_icnn_basket_option}
\end{algorithm}

\subsection{Optimal stopping}
Consider a discrete time grid $0 = t_0 < t_1 <\ldots< t_N = T$ and the $\mathbb{R}^d$-valued Markov chain $\big(X_{t_k}\big)_{0 \le k \le N} := \big(S_{t_k}^1, \ldots, S_{t_k}^d\big)_{0 \le k \le N} $ satisfying $\big\| X_{t_0}\big\|_2 + \cdots + \big\| X_{t_N}\big\|_2 < + \infty$. We denote by $\big(\mathcal{F}_k\big)_{0 \le k \le N}$ the natural completed filtration of the chain. Besides, for any Borel function $f : \mathbb{R}^d  \to \mathbb{R}$, we consider the following operators formally defined by:
\begin{equation}
    \label{transition_markov_chain}
    \mathbf{P}_kf(x) := \mathbb{E}\Big[f\big(X_{t_{k+1}}\big) \big\rvert X_{t_k} = x \Big], \quad x \in \mathbb{R}^d, \quad k \in \{0,\ldots,N-1\}.
\end{equation}

We consider a stochastic game in which you can ask once and only once at some (possibly random) time $t_k$ between $0$ and $T$ to receive the payoff $g_k\big(X_{t_k}\big)$ where $g_k : \mathbb{R}^d \to \mathbb{R}$, the \emph{payoff functions}, are continuous functions with at most linear growth. We denote by $\mathcal{T}_N^{\mathcal{F}}$ the set of all $\big(\mathcal{F}_k\big)_{0 \le k \le N}$-stopping time.

In a risk neutral framework, the price one is willing to pay to participate to this game is given by the following discrete-time optimal stopping problem: 
\begin{equation}
\label{optimal_stopping_pb}
    v_0 := v_0(s_0^1,\ldots,s_0^d) = \mathbb{E}\Big[g_{\tau^*}\big(X_{\tau^*}\big)\Big] = \underset{\tau \in \mathcal{T}_N^{\mathcal{F}}}{\sup} \hspace{0.1cm} \mathbb{E}\Big[g_{\tau}\big(S_\tau^{1},\ldots,S_\tau^{d} \big) \Big],
\end{equation}
where the stopping time $\tau^*$ satisfying \eqref{optimal_stopping_pb} is called \emph{optimal stopping time}. The latter always exists in our setting.

In this paper, we focus on \emph{best-of Bermuda options}. These options allow their holder to buy, at a future date $T$, the maximum between $d$ risky assets and this, at a \emph{strike price} denoted as $K$. That is (with $r > 0$ still being the constant interest rate):
\begin{equation}
    \label{payoff_best_of_berm}
    \mathbb{R}^d \ni (s_1,\ldots,s_d) \mapsto g_k(s_1,\ldots,s_d) := e^{-r t_k}\Big(\underset{1 \le i \le d}{\max} \hspace{0.1cm} s_i - K \Big)_{+}, \quad K > 0.
\end{equation}

A first mathematical solution to this optimal stopping time problem is provided by the \emph{Snell envelope} $\big(V_k\big)_{0 \le k \le N}$ defined by the \emph{Backward Dynamic Programming Principle}:
\begin{equation}
    \label{bdpp_bermuda}
    \left\{
    \begin{array}{ll}
        V_N(X_{t_N}) = g_N\big(X_{t_N}\big),\\
        V_k(X_{t_k}) = \max\Big(g_k\big(X_{t_k}\big), \mathbb{E}\big(V_{k+1}(X_{t_{k+1}}) \big\rvert X_{t_k} \big) \Big), \quad k \in \{ 0,\ldots,N-1\}.
    \end{array}
    \right.
\end{equation}

\noindent
By a straightforward backward induction, one may shows that there exists $N+1$ \emph{value functions} which are Borel functions $v_k : \mathbb{R}^d \to \mathbb{R}$ such that, almost surely $v_k\big(X_{t_k}\big) = V_k\big(X_{t_k}\big)$, and satisfying the functional version of \eqref{bdpp_bermuda} i.e. for $x \in \mathbb{R}^d$:
\begin{equation}
    \label{bdpp_bermuda_val_func}
    \left\{
    \begin{array}{ll}
        v_N(x) = g_N(x),\\
        v_k(x) = \max\big(g_k(x), \mathbf{P}_kv_{k+1}(x) \big), \quad k \in \{ 0,\ldots,N-1\},
    \end{array}
    \right.
\end{equation}

\noindent
where $\mathbf{P}_k$ is the transition \eqref{transition_markov_chain} at time $t_k$. One also shows that 
\begin{equation}
    \tau^* = \min \big\{k: v_k\big(X_{t_k}\big) = g_k\big(X_{t_k}\big) \big\}
\end{equation}
is an optimal stopping time for this problem (in particular it is the greatest). Hence computing numerically the Snell envelope (or at least an approximation of it) yields the value of the optimal stopping problem and a way to estimate the optimal stopping time i.e. the moment at which one can exercise its right to receive the flow/payoff $g_{\tau^*}\big(X_{\tau^*}\big)$.

We now focus on the value function $\mathbb{R}^d \ni (s_1,\ldots,s_d) \mapsto v_k(s_1,\ldots,s_d)$ and aim at showing that this function is convex. To achieve this, we rely on \emph{convex ordering theory} which we now recall the definition.

\begin{definition}[Convex ordering]
Let $U, V \in \mathbb{L}^1_{\mathbb{R}^d}$ with respective distributions $\mu, \nu$. We say that $U$ is dominated for the convex order by $V$, denoted $U \preceq_{cvx} V$, if, for every convex function $f : \mathbb{R}^d \to \mathbb{R}$, one has
$$\mathbb{E}f(U) \le \mathbb{E}f(V)$$

\noindent
or, equivalently, that $\mu$ is dominated for the convex order by $\nu$ (denoted $\mu \preceq_{cvx} \nu$) if, for every convex function $f : \mathbb{R}^d \to \mathbb{R}$,
$$\int_{\mathbb{R}^d}^{} f(\xi) \,d\mu(\xi)  \le \int_{\mathbb{R}^d}^{} f(\xi) \,d\nu(\xi).$$
\end{definition}

Convex ordering-based approaches have recently been used to compare European option prices \cite{Bergenth, BergenthumRüschendorf+2008+53+72}, American option prices \cite{Pagès2016_cvx_ord_path_dep}. It has also been used to prove convexity results in case of American-style options \cite{Pagès2016_cvx_ord_path_dep}.

Without going deeper in convex ordering theory, we may still state the following proposition which may be proved by a straightforward backward induction. For an in depth study of convex ordering, we refer the reader to \cite{RePEc:bpj:strimo:v:26:y:2008:i:1:p:53-72:n:5,Bergenth,Pagès2016_cvx_ord_path_dep}. However, let us state the following generic Proposition.

\begin{Proposition}[Propagation of convexity]
\label{prop_propagation_cvx}
    Assume that the payoff functions $g_k$ are Lipschitz continuous and convex. Assume also the chain $\big(X_{t_k}\big)_{0 \le k \le N}$ propagates the convexity in the sense that for every $k \in \{0,\ldots, N-1\}$:
    \begin{equation}
        \label{hyp_prop_convexity}
        \forall \hspace{0.1cm} f:\mathbb{R}^d \to \mathbb{R}, \text{Lipschitz continuous convex}, \quad \mathbb{R}^d \ni x \mapsto \mathbf{P}_kf(x) \hspace{0.1cm} \text{is Lipschitz continuous and convex}.
    \end{equation}

\noindent
Then, for every $k \in \{0,\ldots, N\}$, the value function $v_k$ is Lipschitz continuous and convex.
\end{Proposition}
The assumption of \emph{propagation of convexity} \eqref{hyp_prop_convexity} is mainly dependent on the Markov chain $\big(X_{t_k}\big)_{0 \le k \le N}$. This is exactly where \emph{convex ordering theory} intervenes, providing mild conditions on the Markov chain under which the assumption of \emph{propagation of convexity} holds. The following example provides with more insight on the latter claim.

\begin{example}
    Suppose $\big(X_{t_k}\big)_{0 \le k \le N}$ is an $\mathbb{R}^d$-valued (martingale) ARCH process i.e.:
    $$X_{t_{k+1}} = X_{t_k} + \sigma\big(X_{t_k}\big) Z_{k+1},$$

    \noindent
    where $(Z_k)_k$ are \emph{i.i.d.} copies of $Z \sim \mathcal{N}(0,\mathbf{I}_q)$. $\sigma : \mathbb{R}^d \to \mathbb{M}_{d, q}(\mathbb{R})$ is the volatility matrix-valued function which is assumed to be convex in the sense that there exist $O_{\lambda, x}, O_{\lambda, y} \in \mathcal{O}\big(q, \mathbb{R} \big)$ such that for any $x,y \in \mathbb{R}^d, \lambda \in [0, 1]$:
    \begin{equation}
        \label{def_convexity_matrix}
        \big(\lambda \sigma(x)O_{\lambda, x} + (1-\lambda)\sigma(y)O_{\lambda, y}\big)\big(\lambda \sigma(x)O_{\lambda, x} + (1-\lambda)\sigma(y)O_{\lambda, y}\big)^\top - \sigma \sigma^\top\big(\lambda x + (1-\lambda)y \big) \in \mathcal{S}^{+}\big(d, \mathbb{R}\big).
    \end{equation}
    Note that, in one-dimension ($d=q=1$), this definition of convexity for matrix-valued functions boils down to regular convexity of $|\sigma|$ (it sufices to set $O_{\lambda, x} = \sign(\sigma(x))$). Moreover, this notion of convexity combined with Proposition \ref{gen_radial_distri} (which uses convex ordering definition) implies that:
    $$\sigma \big(\lambda x + (1-\lambda)y \big) Z \preceq_{cvx} \big(\lambda \sigma(x)O_{\lambda, x} + (1-\lambda)\sigma(y)O_{\lambda, y}\big)Z.$$

    Let $\mathbf{P}_k$ be the transition of the Markov chain $\big(X_{t_k}\big)_{0 \le k \le N}$ at time $t_k$. Then, for any convex function $f: \mathbb{R}^d \to \mathbb{R}$, using Proposition \ref{useful_properties_cvx_ord} yields:
    \begin{align*}
        \mathbf{P}_kf(\lambda x + (1-\lambda)y) &= \mathbb{E}\Big[f\big(\lambda x + (1-\lambda)y + \sigma(\lambda x + (1-\lambda)y)Z  \Big]\\
        &\le \mathbb{E}\Big[f\big(\lambda x + (1-\lambda)y + \big(\lambda \sigma(x)O_{\lambda, x} + (1-\lambda)\sigma(y)O_{\lambda, y}\big)Z  \Big]\\
        &\le \lambda \mathbf{P}_kf(x) + (1-\lambda)\mathbf{P}_kf(y),
    \end{align*}
    where in the last inequality, we used the convexity of $f$ and the fact that $O_{\lambda, x}Z, O_{\lambda, y}Z \sim Z$. This shows a way convex ordering may be used to propagate convexity.
\end{example}

\begin{remark}
    Condition \eqref{def_convexity_matrix} may appears difficult to check at first given a matrix-valued function. But one may show that the quite general class of matrices of the form:
    $$A \cdot \diag\big(\lambda_1(x),\ldots,\lambda_d(x)\big) \cdot O, \quad A \in \mathbb{M}_{d, q}(\mathbb{R}), O \in \mathcal{O}(q, \mathbb{R}),$$

    \noindent
    where $|\lambda_i| : \mathbb{R} \to \mathbb{R}_{+}$ are all convex, is a class of convex matrix-valued functions. Besides, this class includes several models used in practice.
\end{remark}

From Proposition \ref{prop_propagation_cvx}, it turns out that, under some assumptions, among others, the convexity of the payoff functions $g_k$ (which holds for the best-of Bermuda option \eqref{payoff_basket_call}), the function $\mathbb{R}^d \ni x \mapsto \mathbf{P}_kv_{k+1}(x)$ appearing in \eqref{bdpp_bermuda_val_func} is convex so that it may be approximated by our \emph{Convex Network}. Therefore, the value functions given by \eqref{bdpp_bermuda_val_func} may be approximated by functions $\bar{v}_k^{[n]}$ defined by the following backward equation:
\begin{equation}
    \label{approx_ficitive_bdpp_bermuda_val_func}
    \left\{
    \begin{array}{ll}
        \bar{v}_N^{[n]}(x) = g_N(x),\\
        \bar{v}_k^{[n]}(x) = \max\Big(g_k(x), \Tilde{\mathbf{P}}^{\mathcal{NN}, [n]}_k v_{k+1}(x) \Big) \quad \quad k \in \{0, \ldots, N-1\},
    \end{array}
    \right.
\end{equation}
where $\Tilde{\mathbf{P}}^{\mathcal{NN}, [n]}_k v_{k+1}$ is a \emph{Convex Network} $f_{n, L,k}(\cdot; \bm{\mathcal{W}})$ (of the form \eqref{icnn_base_max} or \eqref{icnn_lse}) designed to approximate the convex function $\mathbf{P}_kv_{k+1}$. That is, weights $\bm{\mathcal{W}}$ are such that $\big\| f_{n, L,k}\big(X_{t_k}; \bm{\mathcal{W}}\big) - \mathbf{P}_kv_{k+1}(X_{t_{k+1}}\big) \big\|_2$ is minimized which is equivalent (from Pythagoras Theorem) to solve:
\begin{equation}
    \label{optim_fictive_iccn_bermuda}
    \underset{\bm{\mathcal{W}}}{\inf} \hspace{0.1cm} \Big\| f_{n, L,k}\big(X_{t_k}; \bm{\mathcal{W}}\big) - v_{k+1}(X_{t_{k+1}}\big) \Big\|_2.
\end{equation}

However, the actual function $v_{k+1}$ is unknown rendering the optimization problem \eqref{optim_fictive_iccn_bermuda} unfeasible. For practical numerical solutions, we define the following approximation of the value function:
\begin{equation}
    \label{approx_bdpp_bermuda_val_func}
    \left\{
    \begin{array}{ll}
        \Tilde{v}_N^{[n]}(x) = g_N(x),\\
        \Tilde{v}_k^{[n]}(x) = \max\Big(g_k(x), \Tilde{\mathbf{P}}^{\mathcal{NN}, [n]}_k \Tilde{v}_{k+1}^{[n]}(x) \Big),\quad \quad k \in \{0, \ldots, N-1\},
    \end{array}
    \right.
\end{equation}
where $\Tilde{\mathbf{P}}^{\mathcal{NN}, [n]}_k \Tilde{v}_{k+1}^{[n]}$ is a \emph{Convex Network} $f_{n,L,k}(\cdot; \bm{\mathcal{W}})$ (of the form \eqref{icnn_base_max} or \eqref{icnn_lse}) designed to approximate the convex function $\mathbf{P}_k \Tilde{v}_{k+1}^{[n]}$ (its convexity is proved in the same way as that of $\mathbf{P}_k v_{k+1}$ using convex ordering). That is, with weights solving this time:
\begin{equation}
    \label{optim_iccn_bermuda}
    \underset{\bm{\mathcal{W}}}{\inf} \hspace{0.1cm} \Big\| f_{n,L,k}\big(X_{t_k}; \bm{\mathcal{W}}\big) - \Tilde{v}_{k+1}^{[n]}(X_{t_{k+1}}\big) \Big\|_2.
\end{equation}

Summing up all, solving the optimal stopping time problem with our approach reduces to train $N$ \emph{Convex Networks} at times $t_k, k = 0,\ldots,N-1$ to backwardly solve \eqref{approx_bdpp_bermuda_val_func}. In case $X_0$ is deterministic as it uses to be in Financial Mathematics, only $N-1$ \emph{Convex Networks} need to be trained. Algorithm \ref{algo_icnn_bermuda_option} presents the general procedure.

\begin{algorithm}
    \SetKwFunction{isOddNumber}{isOddNumber}
    \SetKwInOut{KwIn}{Input}
    \SetKwInOut{KwOut}{Output}
    \SetKwRepeat{Do}{do}{while}%

    \KwIn{
        \begin{itemize}[noitemsep]
            \item Target initial assets prices: $(x_0^1,\ldots,x_0^d)$.
            \item Model parameters: $r, \sigma_i, T, \rho_{i, j}$.
            \item Other settings: $K, \alpha_i, M, L, n$.
            \item Batch size: $B$, number of iterations: $I$.
        \end{itemize}
    }

    \KwOut{$v(x_0^1,\ldots,x_0^d)$.}

    \vspace{0.2cm}
    \tcc{Initialize weights $(\bm{\mathcal{W}}_k)_{0 \le k \le N - 1}$ for $N$ \emph{Convex Networks}.}
    
    \vspace{0.2cm}
    \For{$i\gets1$ \KwTo $I$}{

        \vspace{0.2cm}
        Simulate $B$ \emph{i.i.d.} paths $\big(S_{t_k}^{1, (b)},\ldots, S_{t_k}^{d, (b)}\big)_{b \in \{1,\ldots, B\}, k \in \{0, \ldots, N\}}$ of $\big(S_{t_k}^{1},\ldots, S_{t_k}^{d}\big)_{\{0, \ldots, N\}}$ starting from $(x_0^1,\ldots,x_0^d)$.

        \vspace{0.1cm}
        Initialize the value function and the optimal stopping time for each path $b \in \{1,\ldots, B\}$:
        $$v_N^{(b)} = g_N\big(S_T^{1, (b)},\ldots, S_T^{d, (b)}\big), \quad \tau_N^{(b)} = N.$$
        
        \For{$k\gets{N -1}$ \KwTo $0$}{
            Compute the payoff function and estimate the continuation value via the \emph{Convex Network} for each path $b \in \{1,\ldots, B\}$:
            $$G_k^{(b)} := g_k\big(S_{t_k}^{1, (b)},\ldots, S_{t_k}^{d, (b)}\big), \quad C_k^{(b)}(\bm{\mathcal{W}}_k) = f_{n,L, k}\big(S_{t_k}^{1, (b)},\ldots, S_{t_k}^{d, (b)}; \bm{\mathcal{W}}_k\big).$$

            Update weights $\bm{\mathcal{W}}_k$ via Stochastic Gradient Descent to solve:
            $$\underset{\bm{\mathcal{W}}}{\inf} \hspace{0.1cm } \frac{1}{B} \sum_{b = 1}^B \big(C_k^{(b)}(\bm{\mathcal{W}}) -  v_{k+1}^{(b)}\big)^2.$$

            Define the decision function for each path $b \in \{1,\ldots, B\}$:
            $$D_k^{(b)} = \left\{
                                \begin{array}{ll}
                                1 \quad \text{if} \quad G_k^{(b)} \ge C_k^{(b)},\\
                                0 \quad \text{if} \quad G_k^{(b)} < C_k^{(b)}.
                                \end{array}
                            \right.$$

            Update the optimal stopping time and the value function each path $b \in \{1,\ldots, B\}$:
            $$\tau_k^{(b)} = k \cdot D_k^{(b)} + \tau_{k + 1}^{(b)} \cdot \big(1 - D_k^{(b)}\big), \quad v_k^{(b)} = G_k^{(b)} \cdot D_k^{(b)} + v_{k+1}^{(b)}\cdot \big(1 - D_k^{(b)}\big).$$
        }

        Compute estimate of the initial value:
        $$\hat{v}_0 = \frac{1}{B} \sum_{b = 1}^B v_0^{(b)}.$$
    }

    \KwRet{$\hat{v}_0$.}
    \caption{\emph{Convex Network} training: Best-of Bermuda Call option pricing.}
\label{algo_icnn_bermuda_option}
\end{algorithm}

\subsection{Stochastic Optimal Control (SOC)}
We now consider a (discrete time) \emph{Stochastic Optimal Control (SOC)} problem over a (discrete) time grid $0=t_0 < t_1 < \ldots < t_N = T$. Similar to the optimal stopping time problem studied in the previous section, this problem can be think as a stochastic game where a gambler pays an initial amount $v_0$ for the right to take actions $q_k$, at times $t_k$, and receives a reward (which can also be a loss) $g_k$. The latter reward depends on both the gambler's action and a random vector $X_{t_k}$. Here, we consider a constrained version of the game meaning the gambler's actions are restricted to $\big(q_0,\ldots, q_N\big) \in \mathcal{S}$, with $\mathcal{S}$ being the constraint space. Depending on how the set $\mathcal{S}$ is defined, it may prevent for trivial optimal strategies where the gambler may take the action whenever it suits her and do nothing otherwise. A typical and more interesting (in a modelling standpoint) case would be a set $\mathcal{S}$ designed in a way that it might compel the gambler to take an action even when it does not suit her. This framework as described embeds several well-studied real world problems \cite{BarreraEsteve2006NumericalMF, Fleming1975, Bertsekas_stoch_control}. Among these problems, we have decided to study the pricing of \emph{Take-or-Pay} or \emph{Swing} options.

A \emph{Take-or-Pay} (or \emph{Swing}) option is a commodity derivative product which enable its holder to purchase amounts of energy (generally gas delivery contract) $q_k$, at predetermined exercise dates, $t_k, k \in \{0,\ldots,N-1\}$ (thus $N$ exercise dates), until the contract maturity at time $t_N = T$. The purchase price, or \emph{strike price}, is supposed to be fixed for all dates and is denoted as $K$. The holder of the swing contract ought to have a buying/consumption strategy satisfying the following constraints:
\begin{equation}
 \label{swing_vol_constraints}
\left\{
    \begin{array}{ll}
         \underline{q} \le q_{k} \le \overline{q}, \quad k \in \{0,\ldots, N-1\},\\
        \displaystyle Q_{N} := \sum_{k = 0}^{N-1} q_{k} \in \big[\underline{Q}, \overline{Q}\big], \quad Q_0 = 0, \quad 0 \le \underline{Q} \le \overline{Q} < +\infty.
    \end{array}
\right.
\end{equation}

Denote by $F_{t_k}$ the price, at time $t_k$, of the underlying energy instrument. Typically, the latter is generally a forward gas delivery contract which is an agreement paid now for a gas delivery at a future date, allowing the buyer to lock-in the future gas price.

We suppose that there exists a $\mathbb{R}^d$-valued (discrete) Markov chain $\big(X_{t_k}\big)_{0 \le k \le N}$ on a probability space $\Big(\Omega, \big(\mathcal{F}_{k}\big)_{0 \le k \le N}, \mathbb{P}\Big)$ such that for every $k \in \{0,\ldots,N\}, F_{t_k} = f_k\big(X_{t_k}\big)$ where $f_k : \mathbb{R}^d \to \mathbb{R}$ are Borel functions with at most linear growth. $\big(\mathcal{F}_{k}\big)_{0 \le k \le N}$ denotes the natural completed filtration of the chain.

\vspace{0.1cm}
The price $v_0$ one is willing to pay for this contract is given by (we disregard interest rates):
\begin{equation}
    \label{swing_price}
    v_0 := v_0\big(x) = \underset{(q_0, \ldots, q_{N-1}) \in \mathcal{S}}{\sup} \hspace{0.1cm} \mathbb{E}\Bigg[ \sum_{k = 0}^{N-1} q_k\big(F_{t_k} - K\big) \Bigg], \quad X_{t_0} = x \in \mathbb{R}^d,
\end{equation}
where the constraint space $\mathcal{S}$ is designed in light of \eqref{swing_vol_constraints} and defined by:
\begin{equation}
    \mathcal{S} := \Bigg\{(q_k)_{0 \le k \le N-1}, \hspace{0.1cm} q_k : (\Omega, \mathcal{F}_{k}, \mathbb{P}) \mapsto [\underline{q}, \overline{q}] \quad \text{and} \quad \sum_{k = 0}^{N-1} q_{k} \in \big[\underline{Q}, \overline{Q} \big]  \Bigg\}.
\end{equation}

Pricing this contract involves finding an \emph{optimal control} which is an admissible strategy $\big(q_0,\ldots, q_{N-1}\big)$ realizing the supremum in \eqref{swing_price}. This can be approached in two ways. The first approach directly finds $v_0$ by transforming the stochastic optimization involved in \eqref{swing_price} into a parametric one, replacing the controls $q_k$ by well-chosen parametric functions (see \cite{BarreraEsteve2006NumericalMF, lemaire2023swing}). However, in this paper, we adopt a more traditional approach when it comes to \emph{SOC} problems, namely the \emph{Dynamic Programming} approach. To introduce the latter approach, the \emph{SOC} problem \eqref{swing_price} needs to be rewritten seen from a time $t_k$. That is, we dynamically consider:
\begin{equation}
    \label{swing_price_t_k}
    V_k\big(x, Q_k) = \underset{(q_k, \ldots, q_{N-1}) \in \mathcal{S}_k(Q_k)}{\sup} \hspace{0.1cm} \mathbb{E}\Bigg[ \sum_{\ell = k}^{N-1} q_\ell\big(F_{t_\ell} - K\big) \Bigg\rvert X_{t_k} = x \Bigg],
\end{equation}
where the set $\mathcal{S}_k(Q_k)$ is the constraint space seen from time $t_k$ with a cumulative energy consumption $Q_k = \sum_{i = 0}^{k-1} q_i$ (purchased energy until time $t_{k-1}$ with the assumption $Q_0 = 0$). This set is defined by:
\begin{equation}
    \label{vol_constraints_space_t_k}
    \mathcal{S}_k(Q_k) := \Bigg\{(q_{\ell})_{k \le \ell \le N-1}, \hspace{0.1cm} q_{\ell} : (\Omega, \mathcal{F}_{{\ell}}, \mathbb{P}) \mapsto [\underline{q}, \overline{q}] \quad \text{and} \quad \sum_{\ell = k}^{N-1} q_{\ell} \in \Big[(\underline{Q}-Q_k)_{+}, \overline{Q}-Q_k \Big] \Bigg\}.
\end{equation}
Before going any further, let us be more precise in the notations introduced above. The argument $x \in \mathbb{R}^d$ used in \eqref{swing_price} may be misleading since it does not (explicitly) appear in the corresponding formula. Using $x \in \mathbb{R}^d$ as argument refers to the starting value of the Markov chain $\big(X_{t_k}\big)_{0 \le k \le N}$ which drives $F_{t_k}$ appearing in the formula. This implicit notation remains valid in the rest of this section.

\vspace{0.2cm}
Under mild assumptions, we may show that (see \cite{BarreraEsteve2006NumericalMF, Bardou2007WhenAS}) there exist \emph{value functions} $v_k$ coinciding with $V_k$ satisfying the following \emph{backward dynamic programming equation}:
\begin{equation}
 \label{bdpp_swing}
\left\{
    \begin{array}{ll}
        v_{N-1}(x, Q) = \underset{q \in \Gamma_{N-1}(Q)}{\sup} \hspace{0.1cm} q\big(f_{{N-1}}(x) - K\big)\\
        v_{k}(x, Q) = \underset{q \in \Gamma_{k}(Q)}{\sup} \hspace{0.1cm} \Big[q\big(f_{{k}}(x) - K\big) + \big(\mathbf{P}_kv_{k+1}(\cdot, Q+q) \big) (x)  \Big], \quad k \in \{0, \ldots, N-2\},
    \end{array}
\right.
\end{equation}

\noindent
where for every $k \in \{0,\ldots, N-1\}$, $\Gamma_k(Q)$ is the set (in fact a closed interval) of admissible / possible actions at time $t_k$ given that the holder of the contract bought $Q = \sum_{i = 0}^{k-1}$ from the inception of the contract until time $t_{k-1}$ (included). We refer the reader to \cite{lemaire2023swing} for details on this set. Besides, for $k \in \{0,\ldots, N-1\}$, the transition $\mathbf{P}_k$ is formally defined in \eqref{transition_markov_chain}.

As for the optimal stopping time problem, we may use convex ordering theory to prove the convexity of the value function defined in $\eqref{bdpp_swing}$. Especially, assuming a martingale \emph{ARCH} dynamics for the Markov chain $\big(X_{t_k}\big)_{0 \le k \le N}$ and under mild assumptions, it has been shown (see \cite{pagès2024convexorderingstochasticcontrol}) that for every $k$, $\mathbb{R}^d \ni x \mapsto \big(\mathbf{P}_kv_{k+1}(\cdot, Q)\big)(x)$ is convex. Thus, the latter function may be approximated by our \emph{Convex Networks} of the form \eqref{icnn_base_max} or \eqref{icnn_lse}.

We proceed as for the optimal stopping time problem, replacing the convex function $\mathbb{R}^d \ni x \mapsto \big(\mathbf{P}_kv_{k+1}(\cdot, Q)\big)(x)$ by our \emph{Convex Network}. For every $k$, the actual value function $v_k$ is approximated with $\Tilde{v}_k^{[n]}$ backwardly defined by:
\begin{equation}
 \label{bdpp_swing_approx_icnn}
\left\{
    \begin{array}{ll}
        \Tilde{v}_{N-1}^{[n]}(x, Q) = \underset{q \in \Gamma_{N-1}(Q)}{\sup} \hspace{0.1cm} q\big(f_{{N-1}}(x) - K\big)\\
        \Tilde{v}_{k}^{[n]}(x, Q) = \underset{q \in \Gamma_{k}(Q)}{\sup} \hspace{0.1cm} \Big[q\big(f_{{k}}(x) - K\big) + \Tilde{\mathbf{P}}^{\mathcal{NN}, [n]}_k \big(\Tilde{v}_{k+1}^{[n]}(\cdot, Q+q)\big)(x)  \Big], \quad k \in \{0, \ldots, N-2\},
    \end{array}
\right.
\end{equation}
where for every $k \in \{0,\ldots,N-2\}$, $\Tilde{\mathbf{P}}^{\mathcal{NN}, [n]}_k \big(\Tilde{v}_{k+1}^{[n]}(\cdot, Q)\big)$ is a \emph{Convex Network} $f_{n, L, k}^{(Q)}(\cdot; \bm{\mathcal{W}})$ of the form \eqref{icnn_base_max} or \eqref{icnn_lse} such that it solves the following minimization problem:
\begin{equation}
    \label{optim_iccn_swing}
    \underset{\bm{\mathcal{W}}}{\inf} \hspace{0.1cm} \Big\| f_{n, L, k}^{(Q)}\big(X_{t_k}; \bm{\mathcal{W}}\big) - \Tilde{v}_{k+1}^{[n]}(X_{t_{k+1}}, Q\big) \Big\|_2.
\end{equation}

\subsubsection*{Computational efficiency}
Solving \eqref{bdpp_swing_approx_icnn} at a time $t_k$ requires training one \emph{Convex Network} for each possible $Q$ at that time. This may trigger a computationally inefficiency when there are several dates ($N$ is large) and several possible $Q$. An attempt to solve this inefficiency would be to consider $\big(X_{t_k}, Q\big)$ as the state variable. However, this requires to be able to simulate the latter pair whereas we do not have their joint distribution. An efficient approach using \emph{multitask neural networks} (see \cite{yeo2024deepmultitaskneuralnetworks}) has been introduced to address this issue. However, the latter approach solves the \emph{SOC} problem by approximating the optimal action / control rather than the value function, as in our approach. Thus, we can not directly apply this approach and have decided not to focus on this issue which is out of the scope of this paper since it requires a study in its own right. However, one may still improve the complexity of the algorithm by solving \eqref{optim_iccn_swing} at each date $t_k$, for a \q{few} number of reachable $Q$, interpolating the value function when needed.

\subsubsection*{Implementation subtlety}
In light of the preceding discussion on the algorithm's efficiency, one needs to train one \emph{Convex Network} for each time $t_k$, and each cumulative purchased volume $Q \in \mathcal{Q}_k$, where the set $\mathcal{Q}_k$ denotes a finite subset of the set of all reachable cumulative purchased volumes at time $t_k$.

As pointed out in prior works \cite{lemaire2023swing, yeo2024deepmultitaskneuralnetworks}, clearly make the dependence of the value function (or the optimal control) to the cumulative purchased volume could help improve the training. Thus, for each date $t_k, k \in \{0,\ldots,N-1\}$, we consider a modified network as follows (for $Q \in \mathcal{Q}_k$ and $type \in \{\lambda, \max\}$):
\begin{equation}
\label{net_icnn_swing}
    \mathbb{R}^d \ni x \mapsto f_{n, L}^{type}(x, Q; \bm{\mathcal{W}}) := \underbrace{\phi_n^{type} \circ a_L^{\mathcal{W}_L} \circ a_{L-1}^{\mathcal{W}_{L-1}} \circ \cdots \circ a_1^{\mathcal{W}_1}(x)}_{\text{Convex Part}} + \underbrace{\kappa \cdot \mathcal{L}(Q)}_{\text{Adjustment Part}},
\end{equation}
where,
\begin{equation}
    \label{linear_transform_Q}
    \mathcal{L}(Q) := \frac{Q- \underline{Q}}{\overline{Q} - \underline{Q}} \mathbf{1}_{\underline{Q} \neq \overline{Q}} + \frac{Q- \underline{Q}}{\overline{Q}} \mathbf{1}_{\underline{Q} = \overline{Q}},
\end{equation}
and $\kappa \in \mathbb{R}$ is a trainable parameter. The \emph{Convex Part} in the preceding network corresponds to the Convex Network we have studied until now (see \eqref{icnn_base_max} and \eqref{icnn_lse}). The \emph{Adjustment Part}, as its name indicates, allows to adjust the hyperplanes learned as a function of a given cumulative purchased volume. Since this latter part does not depend on our variable of interest, namely $x$, the network built in \eqref{net_icnn_swing} remains convex in $x$. Besides, the linear transformation $\mathcal{L}$ appearing in the \emph{Adjustment Part} comes from studies performed in \cite{lemaire2023swing, yeo2024deepmultitaskneuralnetworks}.

\section{Numerical Experiments}
\label{section_3}
We numerically demonstrate the effectiveness of our approach in the pricing of the three options studied in this paper. To achieve this, we compare our methodology to the well-established ones for each option.

\subsection{Toy example: Noisy convex function}
Before diving into options pricing, let us consider a simple learning problem. Assume that we want to learn a convex function $I \ni x \mapsto f(x)$ given a real interval $I$. To this end, we have at our disposal a training dataset $\big(x_i, g(x_i)\big)_{i \in \{1,\ldots, N\}}$ with $x_i \in I$, and where $g$ is the true function $f$ plus a noise i.e.:
\begin{equation}
    \label{func_g_toy}
    g(x_i) := f(x_i) + \sigma_{\xi} \cdot \xi_i, \quad \xi_i \underset{i.i.d.}{\sim} \mathcal{N}(0,1), \hspace{0.1cm} \sigma_{\xi} > 0.
\end{equation}
In practice, $f(x_i)$ may represents an expectation, and noises $\xi_i$ may come from the Monte Carlo procedure used to compute that expectation as it will be the case in our option pricing problem we described earlier in this paper.

We consider a setting, where the convex function $f$ is given by:
\begin{equation}
    \label{1d-convex-func}
    f(x) := x^2 + 10 \big((e^x - 1) \cdot \mathbf{1}_{x < 0} + x \cdot \mathbf{1}_{x > 0}\big).
\end{equation}
We set $I = [-7,7]$ and $\sigma_\xi = 2$ in this section. The graph of function $f$ is represented in Figure \ref{toy_1d_convex_func} as well as the simulated noisy output $g$ given by the function $f$ plus $\emph{i.i.d.}$ Gaussian noises for each observation $x_i$.

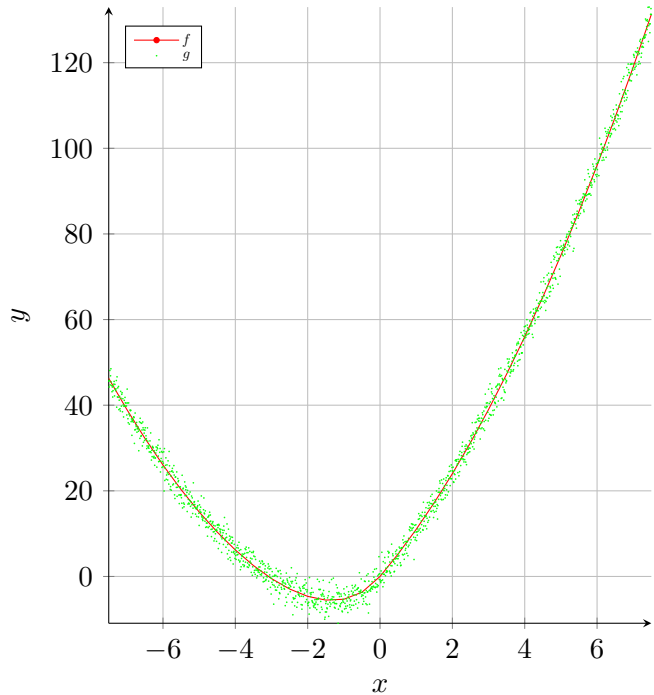
\begin{figure}[ht!]
\centering
\begin{tikzpicture}
\begin{axis}[
    grid=major,
    axis lines = left,
    xlabel = \(x\),
    ylabel = {\(y\)},
    width=0.5\linewidth,
    height=0.35\paperheight,
    legend pos =  north west,
    legend style={nodes={scale=0.5, transform shape},  legend image post style={mark=*}}
]
\addplot[color=red, mark size=1pt] %
	table[x=obs,y=f,col sep=comma]{data/1d_convex_func.csv};
\addlegendentry{$f$};
\addplot[only marks, color = green, mark size=0.1pt]%
	table[x=obs,y=g, col sep=comma]{data/1d_convex_func_sim.csv};
\addlegendentry{$g$};
\end{axis}
\end{tikzpicture}
\caption{\textit{Toy example: convex function $f$ with its noisy version $g$.}}
\label{toy_1d_convex_func}
\end{figure}

We then train our \emph{Convex Network} to learn the true convex function $I \ni x \mapsto f(x)$, given observed outputs $g(x_i)$ for points $x_i \in I$. This is performed by building a network of the form \eqref{icnn_base_max} or \eqref{icnn_lse}, and train it in a way that it minimizes the mean squared error to the output $g$. For the training, we use $n=32$ units per layer, and 100 batches of size $4096$ of points $x_i$ uniformly drawn from $I$. We also use Adam as optimizer and the coefficient $c$ (in Remark \ref{choix_lambda}) is set to 10. Besides, for numerical performance, we scale the input data using the transformation:
\begin{equation}
    \label{scaling_x}
    \Tilde{x}_i := \frac{x_i + 7}{14}.
\end{equation}
To set the learning rate, we follow the following schedule rule ($i$ denoting the iterations):
\begin{equation}
    \label{schedule_lr}
    \gamma_i = 1e^{-3} \quad \text{for} \quad i \le 100 \quad \text{and} \quad \gamma_i = \max\big(1e^{-5}, 0.95  \gamma_{i-1}\big) \quad \text{for} \quad i > 100.
\end{equation}

\noindent
This scheduling is due to the fact that we observed a big variability in the loss for large number of iterations and for the \q{Scrambling} versions of our \emph{Convex Network}. The training loss functions are depicted in Figure \ref{toy_iccn_appr_loss}.

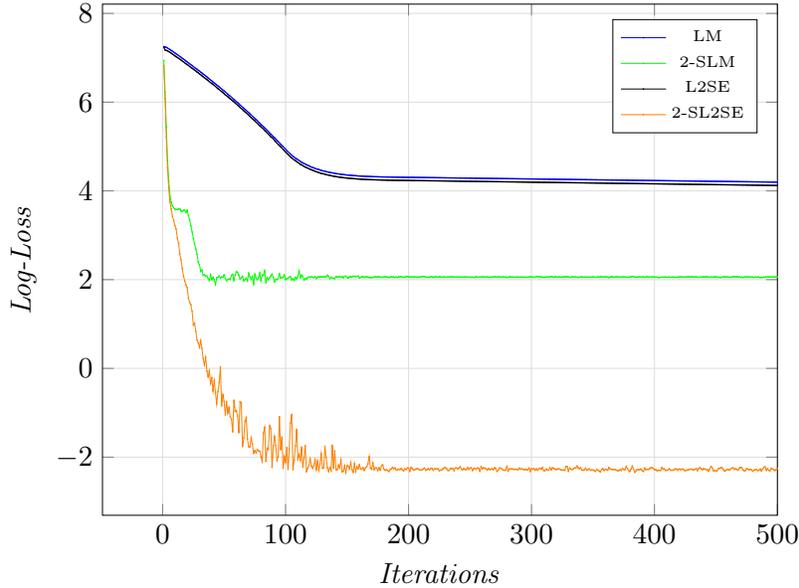
\begin{figure}[ht!]
\centering
\begin{tikzpicture}
\begin{axis}[
	xlabel= \textit{Iterations},
	ylabel=\textit{Log-Loss},
	grid=both,
        xmax = 500,
	minor grid style={gray!25},
	major grid style={gray!25},
	width=0.6\linewidth,
	height=0.3\paperheight,
	line width=0.1pt,
	mark size=0.1pt,
	legend pos =  north east,
        legend style={font=\tiny}
	]
\addplot[color=blue, mark=*] %
	table[x=iter,y=lm,col sep=comma]{data/toy_loss.csv};
\addlegendentry{LM};
\addplot[color=green, mark=*] %
	table[x=iter,y=2slm,col sep=comma]{data/toy_loss.csv};
\addlegendentry{2-SLM};
\addplot[color=black, mark=*] %
	table[x=iter,y=l2se,col sep=comma]{data/toy_loss.csv};
\addlegendentry{L2SE};
\addplot[color=orange, mark=*] %
	table[x=iter,y=2sl2se,col sep=comma]{data/toy_loss.csv};
\addlegendentry{2-SL2SE};
\end{axis}
\end{tikzpicture}
\caption{\textit{Toy example: approximation of a noisy convex function by our Convex Networks.}}
\label{toy_iccn_appr_loss}
\end{figure}
Figure \ref{toy_iccn_appr_loss} highlights the power of combining a \emph{Scrambled} network with the regularization of the maximum function via the \emph{LogSumExp} function. The \emph{2-ScrambledLogSumExp} (2-SL2SE) network appears to outperform the other networks, with a more rapid decrease in the loss function. This is further demonstrated in the test phase (see Figure \ref{toy_iccn_approx}, where the true function $f$ is compared to its approximation by our \emph{Convex Network}), where we evaluate the model using 100 uniformly drawn points from $I$, different from those used during training phase. The results show that our \emph{Convex Network} effectively recovers the true convex function $f$.
\newpage

\begin{figure}[ht!]
\centering
\begin{tikzpicture}
\begin{axis}[
    grid=major,
    axis lines = left,
    xlabel = \(x\),
    ylabel = {\(y\)},
    width=0.45\linewidth,
    height=0.25\paperheight,
    legend pos =  north west,
]
\addplot[color=red, mark size=1pt] %
	table[x=obs,y=f,col sep=comma]{data/toy_loss_f_hat_non_scale.csv};
\addlegendentry{$f$};
\addplot[color=blue, mark size=1pt] %
	table[x=obs,y=f_hat_2sl2se,col sep=comma]{data/toy_loss_f_hat_non_scale.csv};
\addlegendentry{$2-SL2SE$};
\end{axis}
\end{tikzpicture}
~\vspace{0.4cm}
\begin{tikzpicture}
\begin{axis}[
    grid=major,
    axis lines = left,
    xlabel = \(x\),
    ylabel = {Relative error},
    width=0.45\linewidth,
    height=0.25\paperheight,
    enlarge y limits=true,
	]
\addplot[color=green, mark=*] %
	table[x=obs,y=diff,col sep=comma]{data/toy_loss_f_hat_non_scale.csv};
\end{axis}
\end{tikzpicture}
\caption{\textit{Toy example: 2-SL2SE approximation of a noisy convex function (on the left). The relative error is given on the right graphic. The relative error is given by $\frac{|\hat{f} - f|}{f}$, where $\hat{f}$ is the approximation of $f$ by the network 2-SL2SE. The formula of the relative error explains why the error is large for value of $f$ close to 0.}}
\label{toy_iccn_approx}
\end{figure}
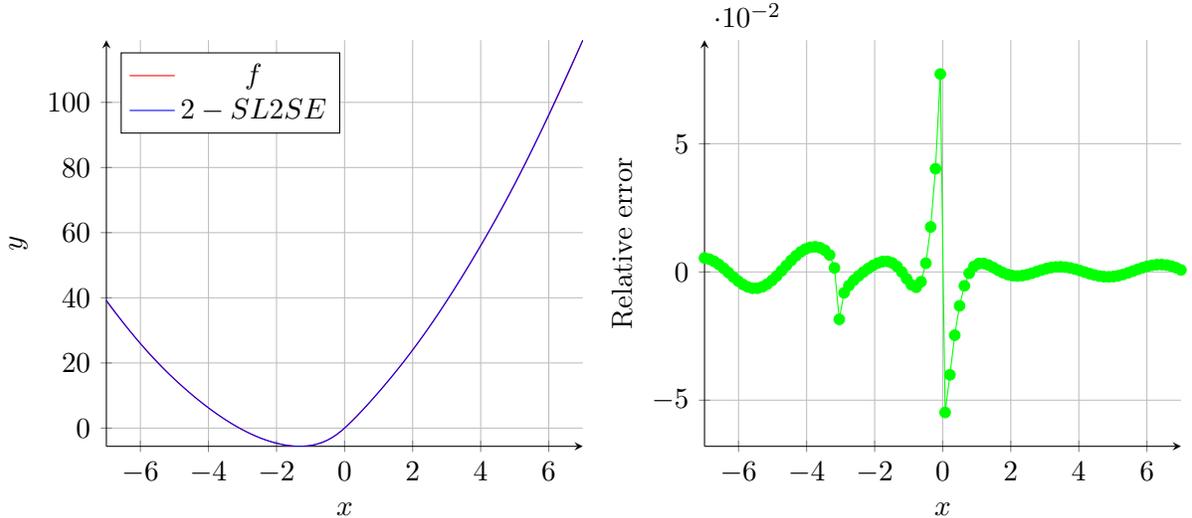

\subsection{Basket call option}
For the first option pricing problem, we consider a basket call option (with payoff \eqref{payoff_basket_call}) along with a multidimensional Black \& Scholes model \eqref{asset_price_bs_model}. For this derivative product, we compare prices given by our \emph{Convex Network} to those given by the improved Monte Carlo procedure \eqref{estimator_with_cv_basket_call} with a sample of size $M = 10^6$. Disregarding the chosen \emph{Convex Network} architecture, we set the number of units per layer to $n = 32$. For the training procedure, we use a sample of size 38400 with batches of size 64, and Adam as optimizer. We also set $R = 20$ along with the input distribution discussed in \eqref{n_strat_s0_input}. That is, in the training procedure of our \emph{Convex Network}, we learn prices of basket call options as a function of the starting price $(s_0^1,\ldots,s_0^d)$ of the underlying asset which lies within $\mathcal{I}(x_0, R) := \prod_{i=1}^d \big[x_0^i - R, x_0^i + R\big]$ for some $x_0 := (x_0^1,\ldots,x_0^d) \in \big(\mathbb{R}_{+}^*\big)^d$. Then, in the test phase, we want to show that, even by learning from few points (here 38400 points) in the set $\mathcal{I}(x_0, R)$, our \emph{Convex Network}, by filling the convex price's curve, can accurately retrieve prices of basket options, where the underlying asset price starts from any value of $\mathcal{I}(x_0, R)$, especially for points not already seen in the training phase.

To set the learning rate, we follow the schedule rule described in \eqref{schedule_lr}. Besides, the coefficient $c$ driving the parameter $\lambda$ in the \emph{LogSumExp} function (see discussion in Remark \ref{choix_lambda}) is set to 20 for $d\le 2$ and to 40 otherwise.

For the underlying asset, we consider the following settings: $x_0^i = 100 - i, \sigma_i = 0.2 + 0.008 i, \alpha_i = 1/d, \delta_i = 0$ for $i \in \{1,\ldots, d\},  r = 0.06, K = 80, t = 0.5$ (6 months). The instantaneous correlation is set to $\rho_{i, j} = \rho \mathbf{1}_{i \neq j} + \mathbf{1}_{i=j}$ with $\rho \in \big(-\frac{1}{d - 1}, 1\big)$. The loss function during the training phase is depicted in Figures \ref{comparaison_loss_par_arch_zero_corr}, \ref{comparaison_loss_par_arch_non_zero_corr}, where the testing phase is performed on $3 \cdot 10^6$ points randomly chosen in $\mathcal{I}(x_0, R)$. 

\begin{figure}[!h]
\centering
\begin{tikzpicture}
\begin{axis}[
	xlabel= \textit{Iterations},
	ylabel=\textit{Log-Loss},
        xmax = 250,
	grid=both,
	minor grid style={gray!25},
	major grid style={gray!25},
	width=0.5\linewidth,
	height=0.25\paperheight,
        mark= star,
	line width=0.1pt,
	mark size=0.1pt,
	legend pos =  north east,
        legend style={font=\tiny}
	]
\addplot[color=blue, mark=*] %
	table[x=iter,y=lm,col sep=comma]{data/loss_comp_vanilla_dim_2.csv};
\addlegendentry{LM};
\addplot[color=green, mark=*] %
	table[x=iter,y=2slm,col sep=comma]{data/loss_comp_vanilla_dim_2.csv};
\addlegendentry{2-SLM};
\addplot[color=black, mark=*] %
	table[x=iter,y=l2se,col sep=comma]{data/loss_comp_vanilla_dim_2.csv};
\addlegendentry{L2SE};
\addplot[color=orange, mark=*] %
	table[x=iter,y=2sl2se,col sep=comma]{data/loss_comp_vanilla_dim_2.csv};
\addlegendentry{2-SL2SE};
\end{axis}
\end{tikzpicture}%
~
\begin{tikzpicture}
\begin{axis}[
	xlabel= \textit{Iterations},
	ylabel=\textit{Log-Loss},
	grid=both,
        xmax = 250,
	minor grid style={gray!25},
	major grid style={gray!25},
	width=0.5\linewidth,
	height=0.25\paperheight,
	line width=0.1pt,
	mark size=0.1pt,
	legend pos =  north east,
        legend style={font=\tiny}
	]
\addplot[color=blue, mark=*] %
	table[x=iter,y=lm,col sep=comma]{data/loss_comp_vanilla_dim_5.csv};
\addlegendentry{LM};
\addplot[color=green, mark=*] %
	table[x=iter,y=2slm,col sep=comma]{data/loss_comp_vanilla_dim_5.csv};
\addlegendentry{2-SLM};
\addplot[color=black, mark=*] %
	table[x=iter,y=l2se,col sep=comma]{data/loss_comp_vanilla_dim_5.csv};
\addlegendentry{L2SE};
\addplot[color=orange, mark=*] %
	table[x=iter,y=2sl2se,col sep=comma]{data/loss_comp_vanilla_dim_5.csv};
\addlegendentry{2-SL2SE};
\end{axis}
\end{tikzpicture}
\caption{\textit{Logarithm of the training losses per Convex Network architecture with $\rho = 0$ ($d = 2$ on the left and $d=5$ on the right).}}
\label{comparaison_loss_par_arch_zero_corr}
\end{figure}
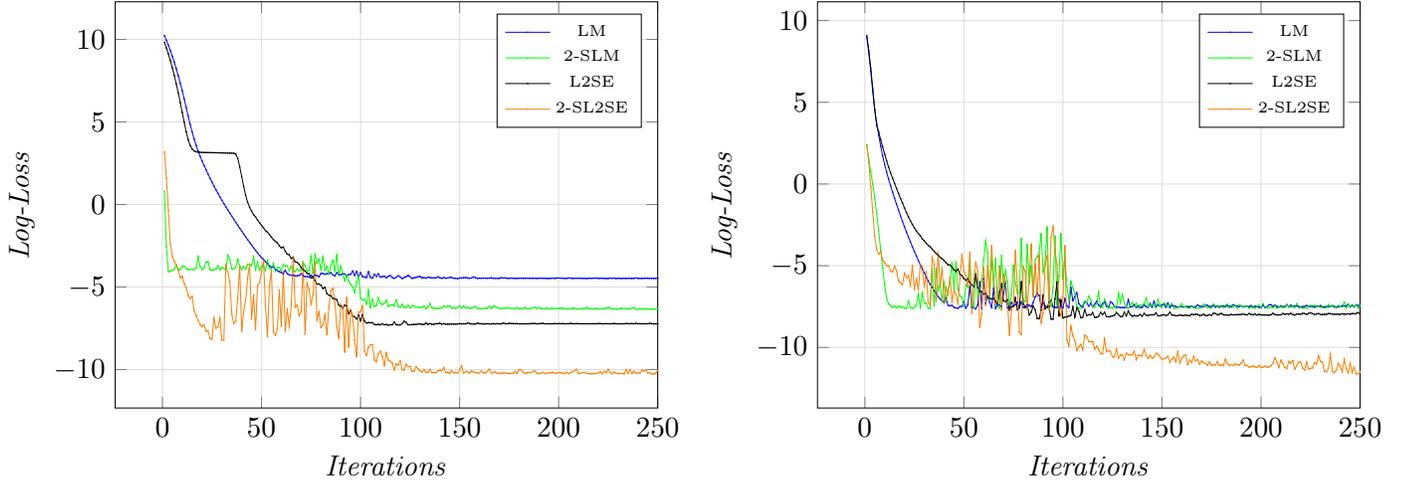

\newpage
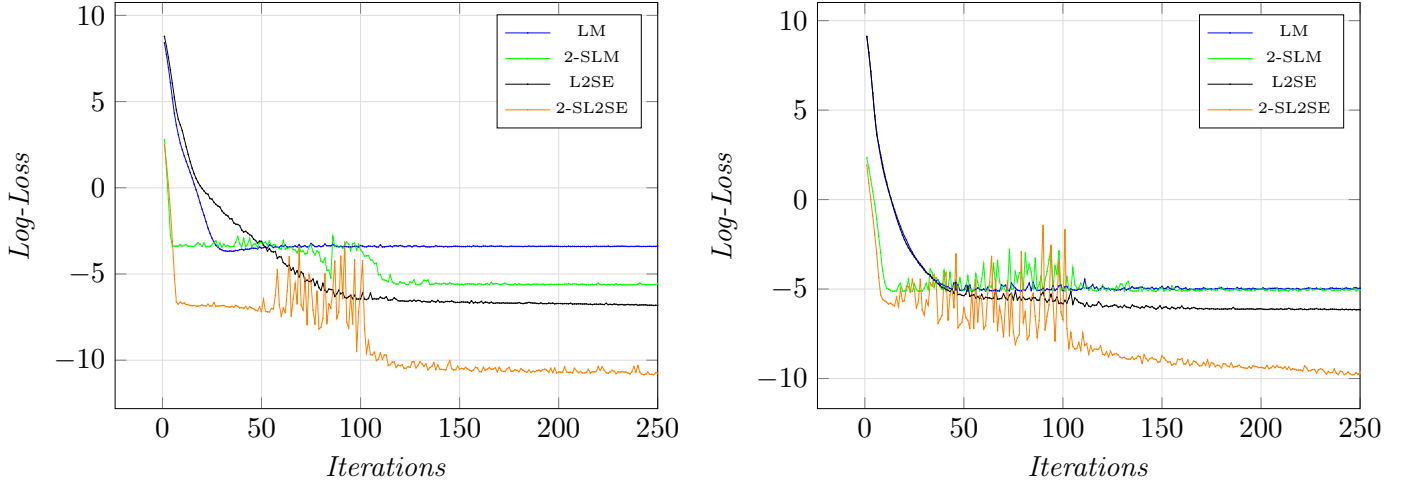
\begin{figure}[!h]
\centering
\begin{tikzpicture}
\begin{axis}[
	xlabel= \textit{Iterations},
	ylabel=\textit{Log-Loss},
	grid=both,
        xmax = 250,
	minor grid style={gray!25},
	major grid style={gray!25},
	width=0.5\linewidth,
	height=0.25\paperheight,
        mark= star,
	line width=0.1pt,
	mark size=0.1pt,
	legend pos =  north east,
        legend style={font=\tiny}
	]
\addplot[color=blue, mark=*] %
	table[x=iter,y=lm,col sep=comma]{data/loss_comp_vanilla_dim_2_corr.csv};
\addlegendentry{LM};
\addplot[color=green, mark=*] %
	table[x=iter,y=2slm,col sep=comma]{data/loss_comp_vanilla_dim_2_corr.csv};
\addlegendentry{2-SLM};
\addplot[color=black, mark=*] %
	table[x=iter,y=l2se,col sep=comma]{data/loss_comp_vanilla_dim_2_corr.csv};
\addlegendentry{L2SE};
\addplot[color=orange, mark=*] %
	table[x=iter,y=2sl2se,col sep=comma]{data/loss_comp_vanilla_dim_2_corr.csv};
\addlegendentry{2-SL2SE};
\end{axis}
\end{tikzpicture}%
~
\begin{tikzpicture}
\begin{axis}[
	xlabel= \textit{Iterations},
	ylabel=\textit{Log-Loss},
	grid=both,
        xmax = 250,
	minor grid style={gray!25},
	major grid style={gray!25},
	width=0.5\linewidth,
	height=0.25\paperheight,
	line width=0.1pt,
	mark size=0.1pt,
	legend pos =  north east,
        legend style={font=\tiny}
	]
\addplot[color=blue, mark=*] %
	table[x=iter,y=lm,col sep=comma]{data/loss_comp_vanilla_dim_5_corr.csv};
\addlegendentry{LM};
\addplot[color=green, mark=*] %
	table[x=iter,y=2slm,col sep=comma]{data/loss_comp_vanilla_dim_5_corr.csv};
\addlegendentry{2-SLM};
\addplot[color=black, mark=*] %
	table[x=iter,y=l2se,col sep=comma]{data/loss_comp_vanilla_dim_5_corr.csv};
\addlegendentry{L2SE};
\addplot[color=orange, mark=*] %
	table[x=iter,y=2sl2se,col sep=comma]{data/loss_comp_vanilla_dim_5_corr.csv};
\addlegendentry{2-SL2SE};
\end{axis}
\end{tikzpicture}
\caption{\textit{Logarithm of the training losses per Convex Network architecture with $\rho = 0.4$ ($d = 2$ on the left and $d=5$ on the right).}}
\label{comparaison_loss_par_arch_non_zero_corr}
\end{figure}

Overall, we observe that, disregarding the problem dimension, the \textit{2-SL2SE} network outperforms the others, with a faster decrease in the loss function. In high dimension, the network \textit{2-SL2SE} distinguished itself from the other networks, which exhibit similar performance.

\vspace{0.2cm}
To assess, the accuracy of our method, following its performance, we take the \textit{2-SL2SE} architecture and consider the following five test cases: $s_0 = (85-i + 6j)_{i=1:d}$ for $j \in \{1, \ldots, 5\}$. Results attesting the accuracy of our method are recorded in Table \ref{table_comparison_price_basket}, where \textit{MC+CV} stands for the price given by the Monte Carlo procedure improved with the use of control variate as explained in \eqref{estimator_with_cv_basket_call}.

\newpage

\begin{table}[ht]
\centering
\begin{tblr}{
    hlines,
    vlines,
    colspec={|cc|cc|cc|},
    cell{1}{1} = {r=2}{c},
    cell{1}{3,5} = {c=2}{c},
}
     & & $\rho = 0$ & & $\rho = 0.4$ & \\
     & $j$ & $2-SL2SE$ & $MC+CV$ & $2-SL2SE$ & $MC+CV$\\
     \hline
    \SetCell[r=5]{} $d=2$ &1 & 12.238  & 12.236 ([12.234, 12.237])& 12.545  & 12.545 ([12.544, 12.546])\\
    &2 & 17.955  & 17.953 ([17.952, 17.955])& 18.100 & 18.097 ([18.096, 18.098])\\
    &3 & 23.877  & 23.882 ([23.880, 23.884])& 23.938 & 23.941 ([23.940, 23.942])\\
    & 4& 29.863  & 29.867 ([29.866, 29.869])& 29.890 & 29.894 ([29.892, 29.895])\\
    &5 & 35.868 & 35.865 ([35.863, 35.867])& 35.884  & 35.881 ([35.880, 35.883])\\
    \SetCell[r=5]{} $d=5$ &1 & 10.455 & 10.456 ([10.455, 10.457])& 11.067 & 11.069 ([11.069, 11.070])\\
    &2 & 16.371 & 16.370 ([16.369, 16.372])& 16.613 & 16.611 ([16.610, 16.612])\\
    &3 & 22.363 & 22.364 ([22.363, 22.366])& 22.471 & 22.471 ([22.470, 22.472])\\
    &4 & 28.367  & 28.364 ([28.362, 28.365])& 28.436 & 28.438 ([28.436, 28.439])\\
    &5 & 34.373 & 34.364 ([34.362, 34.365])& 34.435 & 34.433 ([34.431, 34.434])\\
    \SetCell[r=5]{} $d=20$ &1 & 3.295 & 3.285 ([3.284, 3.287])& 6.020 & 6.062 ([6.060, 6.064])\\
    &2 & 8.874 & 8.872 ([8.871, 8.874])& 10.467& 10.460 ([10.459, 10.462])\\
    &3 & 14.863 & 14.863 ([14.862, 14.864])& 15.712 & 15.695 ([15.693, 15.696])\\
    &4 & 20.867 & 20.863 ([20.862, 20.864])& 21.381 & 21.386 ([21.384, 21.388])\\
    &5 & 26.872 & 26.863 ([26.861, 26.864])& 27.234  & 27.282 ([27.281, 27.284])\\
\end{tblr}
\caption{\textit{Prices of basket call options with our Convex Network. The parameter $\mu$ allows to generate initial prices by using $s_0^i = 85-i + 6j$ for $j \in \{1, \ldots, 5\}$. The training phase of our Convex Network is performed with 1500 iterations.}}
\label{table_comparison_price_basket}
\end{table}

Table \ref{table_comparison_price_basket} showcases the ability of our \emph{Convex Network} to retrieve the basket option prices. This also highlights one of the advantage of our method: even for high dimension (for example for $d=20$ in Table \ref{table_comparison_price_basket}), a small input size is sufficient. This is due to the fact that, our \emph{Convex Network} integrates, by construction, the a priori convexity property of the price of the basket option. Thus filling this convex shape, only few points are sufficient to recover the whole curve.

\subsection{Bermudan option}
The second example is a Bermudan call option on the multidimensional Black \& Scholes model. In the latter model, we consider two different settings which are those considered for the \emph{Deep Optimal Stopping (DOS)} method (see \cite{JMLR:v20:18-232}):

\vspace{0.2cm}
\noindent
\textbf{Symmetric case.}
\noindent
$s_0^i = s_0, \sigma_i = \sigma, \delta_i = \delta$ for all $i \in \{1, \ldots, d\}$ and $\rho_{i, j} = \mathbf{1}_{i = j}$.

\vspace{0.2cm}
\noindent
\textbf{Asymmetric case.}
\noindent
$s_0^i = s_0, \sigma_i = \sigma, \delta_i = \delta$ for all $i \in \{1, \ldots, d\}$ and $\rho_{i, j} = \mathbf{1}_{i = j}$, but unlike the symmetric case, we consider different volatilities given by $\sigma_i = 0.08 + 0.32(i-1)/(d-1)$ for $i \in \{1, \ldots,d\}$ when $d \le 5$. For $d > 5$, we set $\sigma_i = 0.1 + i/(2d)$ for $i \in \{1, \ldots,d\}$.

\vspace{0.2cm}
The time grid is given by $t_k = \frac{kT}{N}$ for $k \in \{0,\ldots,N\}$, where we set $T = 3, N = 9$. The strike price is set to $K = 100$. For the training phase, we use 8 batches of size 1024, and 5000 iterations. The learning rate is set to $1e^{-4}$. We also $n=64$ units per layers, and set the constant $c$ (in Remark \eqref{choix_lambda}) to 40. For the test phase, we use a dataset of size $3\cdot10^6$. Results are recorded in Table \ref{table_bermuda_comp}.

\newpage

\begin{table}[ht!]
    \centering
\begin{tblr}{hlines,vlines,colspec={|c|c|ccc|ccc|}, cell{1}{3,6} = {c=3}{c}}
\hline
        & & Symmetric & & & Asymmetric & & \\
     $d$ & $s_0$ & $2-SL2SE$ & $DOS_L$ & $DOS_U$& $2-SL2SE$ & $DOS_L$ & $DOS_U$ \\
     \hline
     $2$ & 90  & 8.04 & 8.072 & 8.075 & 14.28  & 14.325 & 14.352\\
     $2$ & 100  & 13.84 & 13.895 & 13.903& 19.67 & 19.802 & 19.813\\
     $2$ & 110  & 21.27 & 21.353 & 21.346 & 27.07 & 27.170 & 27.147\\
     $3$ & 90  & 11.22 & 11.290 & 11.283& 19.03 & 19.093 & 19.089\\
     $3$ & 100  & 18.63 & 18.690 & 18.691& 26.61 & 26.680 & 26.684\\
     $3$ & 110  & 27.51 & 27.564 & 27.581& 35.79& 35.842 & 35.817\\
     $5$ & 90 & 16.56 & 16.648 & 16.640& 27.56& 27.662 & 27.662\\
     $5$ & 100  & 26.07 & 26.156 & 26.162& 37.94& 37.976 & 37.995\\
     $5$ & 110  & 36.69 & 36.766 & 36.777& 49.43& 49.485 & 49.513\\
     $10$ & 90  & 26.15 & 26.208 & 26.272& 85.81& 85.937 & 86.037\\
     $10$ & 100  & 38.25 & 38.321 & 38.353& 104.287& 104.692 & 104.791\\
     $10$ & 110  & 50.76 & 50.857 & 50.914& 123.36& 123.668 & 123.823\\
     $20$ & 90  & 37.70 & 37.701 & 37.903& 125.487 & 125.916 & 126.275\\
     $20$ & 100  & 51.52 & 51.571 & 51.765& 149.179& 149.587 & 149.970\\
     $20$ & 110  & 65.50 & 65.494 & 65.762& 172.76& 173.262 & 173.809\\
     $30$ & 90  & 44.73 & 44.797 & 45.110& 153.964 & 154.486 & 154.913\\
     $30$ & 100  & 59.48 & 59.498 & 59.820& 180.745& 181.275 & 181.898\\
     $30$ & 110 & 74.20 & 74.221 & 74.515& 207.592& 208.223 & 208.891\\
\end{tblr}
\caption{\textit{Bermudan option pricing: comparison of our Convex Network with the Deep Optimal stopping method. $DOS_L$ and $DOS_U$ stand for the lower bound and the upper bound given by the latter method. For the symmetric and the asymmetric cases, we set $r=5\%, \sigma=20\%, \delta = 10\%$.}}
\label{table_bermuda_comp}
\end{table}

Table \ref{table_bermuda_comp} proves the effectiveness of our \emph{Convex Network} which gives prices that are very close (relative errors are less than $1\%$) to those given by the \emph{Deep Optimal Stopping} method (see \cite{JMLR:v20:18-232}). Our method works well disregarding the dimension. Besides, it worth noting that our method gives, \q{for free}, the \emph{Greeks} of the Bermudan option. Indeed, by construction, the weights $w^{*}$ of the \emph{Convex Network} which realizes the maximum in \eqref{max_affine_functions} (\emph{activated hyperplanes}) should converge towards the subgradient of the convex function of interest (in this case, the price of the Bermudan option). Thus, by performing the pricing with our \emph{Convex Network}, one can easily retrieve the sensitivity of the Bermudan option price with respect to the underlying asset price (often called \emph{Delta}) by just looking at the weight $w^{*}$ obtained at the end of the training phase.

\subsection{Take-or-Pay option}
For the last example, we consider a Swing or Take-or-Pay contract with $N = 31$ exercise dates. At each exercise date, the holder of the contract is allowed to buy, at price $K = 20$, a volume of gas which is subject to constraints ruled by integers: $\underline{q}=0, \overline{q}=1$ (one can always reduce to this setting, see \cite{Bardou2009OptimalQF}).

For the diffusion model, we consider the following dynamics:
\begin{equation}
    \frac{dF_{t, T}}{F_{t, T}} = \sigma e^{-\alpha (T-t)} dW_t, \quad t \le T,
\end{equation}
where $(W_t, t \ge 0)$ is a one-dimensional standard Brownian motion. Here, we set $\alpha = 4, \sigma = 0.7$, and $F_{0, t} = 20$ for all $t \ge 0$. We assume that the underlying gas price at each exercise date $t_k$ for $k \in \{0,\ldots, N\}$ is modeled by:
\begin{equation}
    F_{t_k} := F_{t_k, t_k} = F_{0, t_k} e^{\sigma X_{t_k} - \frac{1}{2}\lambda_k^2}, \quad X_{t_k} = \int_{0}^{t_k} e^{-\alpha (t_k - s)} \,dW_s, \quad \lambda_k^2 = \frac{\sigma^2}{2\alpha}\big(1 - e^{-2 \alpha t_k} \big).
\end{equation}

For the training phase, we use 5 batches of size $4096$, and 2000 iterations. The learning rate is set to $1e^{-3}$. Each hidden layer of the \emph{Convex Network} has $n = 32$ units, and we set $c = 20$ for the scaling parameter (see Remark \ref{choix_lambda}). For the test phase, the price of the Swing option is computed using $2\cdot10^6$ simulations of path $(F_{t_k})_{0 \le k \le N}$ along which the \emph{Convex Network}, already trained, is used to estimate the optimal control, and then the price of the Swing option. Results are recorded in Table \ref{tab_results_swing}.

\begin{table}[ht!]
    \centering
\begin{tblr}{hlines,colspec={|c|c|c|c|}}
\hline
    $\underline{Q}$ & $\overline{Q}$ & 2-SL2SE & Benchmark\\
     \hline
      20 & 25 & 8.35 & 8.36\\
     20 & 30 & 13.98 & 14.01\\
    20 & 22 & 4.51 & 4.50\\
\end{tblr}
\caption{\textit{Swing pricing results. The benchmark refers to the deep neural network based method named NN-Strat developed in \cite{lemaire2023swing}.}}
\label{tab_results_swing}
\end{table}
Table \ref{tab_results_swing} showcases the efficiency of our \emph{Convex Network} in solving a Stochastic Optimal Control Problem with a convexity feature, namely the pricing of Swing or Take-or-Pay options.

\bibliographystyle{alpha}
\bibliography{biblio.bib}

\appendix

\section{Some useful results}

\begin{theorem}[Zador’s Theorem (see Theorem 5.2 in \cite{Pagès2018Quantif})]
    \label{Zador_thm}
    Consider a probability space $\big(\Omega, \mathcal{A}, \mathbb{P}\big)$. Let $\alpha_N := (x_1,\ldots,x_N) \in \big(\mathbb{R}^d)^N$. Then, one has:
    \begin{enumerate}[label=\roman*.]
\item (\emph{Sharp rate} \cite{LuschgyQuantif}). Let $X \in \mathbb{L}_{\mathbb{R}^d}^{p+\delta}(\mathbb{P})$ for some $\delta > 0$. Let $\mathbb{P}_X(d\xi) = \varphi(\xi)\lambda_d(d\xi) + \nu(d\xi)$, where $\nu \,\bot\ \lambda_d$, i.e. is singular with respect to the Lebesgue measure $\lambda_d$ on $\mathbb{R}^d$. Then, there is a positive constant $C_{p, d}$ such that:
$$\lim\limits_{N \rightarrow +\infty} \hspace{0.1cm} N^{\frac{1}{d}} \cdot \underset{x \in (\mathbb{R}^d)^N}{\min} \hspace{0.1cm} \big\|X-\hat{X}^{\alpha_N }  \big\|_p = C_{p, d}\Bigg[\int_{\mathbb{R}^d}^{} \varphi^{\frac{d}{d+p}} \, \mathrm{d}\lambda_d \Bigg]^{\frac{1}{p} + \frac{1}{d}}.$$

\item (\emph{Non-asymptotic upper-bound} \cite{LuschgyPagèsQuantif}). Let $\delta > 0$. There exists a positive constant $C_{d, p, \delta}$ such that, for every $\mathbb{R}^d$-valued random vector $X$,
$$\forall N \ge 1, \quad \underset{x \in (\mathbb{R}^d)^N}{\min} \hspace{0.1cm} \big\|X-\hat{X}^{\alpha_N}  \big\|_p  \le C_{d, p, \delta} \cdot \sigma_{p+\delta}(X) N^{-\frac{1}{d}},$$
where, for $r > 0$, $\sigma_r(X) = \underset{a \in \mathbb{R}^d}{\min} \hspace{0.1cm} \big\|X - a\big\|_r \le + \infty$.
\end{enumerate}
\end{theorem}

\begin{Proposition}[see \cite{jourdain:hal-02304190}]
\label{gen_radial_distri}
Let $Z \in \mathbb{R}^q$ be a $q$-dimensional random vector having a radial distribution in the sense that
$$\forall \quad O \in \mathcal{O}\big(q, \mathbb{R}\big), \quad OZ \sim Z.$$

\noindent
Let $A, B \in \mathbb{M}_{d, q}(\mathbb{R})$. Then, we have the following equivalence
\begin{equation*}
BB^\top - AA^\top \in \mathcal{S}^{+}\big(d, \mathbb{R}\big) \iff AZ \preceq_{cvx} BZ.
\end{equation*}
\end{Proposition}

\begin{Proposition}
\label{useful_properties_cvx_ord}
Let $f : \mathbb{R}^d \to \mathbb{R}$ be a convex function with linear growth. Define the operator $\mathcal{T}$ by:
$$\mathbb{R}^d \times \mathbb{M}_{d, q}\big(\mathbb{R}\big) \ni  (x, A) \mapsto \big(\mathcal{T} f\big)(x, A) := \mathbb{E}f\big(x+AZ\big).$$.
\begin{enumerate}[label=\roman*.]
\item \label{right_O_inv_tf} For all $x \in \mathbb{R}$, $\mathcal{T} f(x, \cdot)$ is right $\mathcal{O}\big(q, \mathbb{R}\big)$-invariant i.e., 
$$\forall O \in \mathcal{O}\big(q, \mathbb{R}\big), \quad \mathcal{T} f(x, AO) = \mathcal{T} f(x, A).$$

\item \label{tf_convex} $\mathcal{T} f(\cdot, \cdot)$ is convex.

\item \label{croiss_mat} Define a pre-order on matrix spaces as follows:
$$A \preceq B \iff BB^\top - AA^\top \in \mathcal{S}^{+}\big(d, \mathbb{R}\big).$$

Then, for all $x \in \mathbb{R}$, $\mathcal{T} f(x, \cdot)$ is non-decreasing with respect to the preceding pre-order on matrices i.e., $A \preceq B \implies \mathcal{T} f(x, A) \le \mathcal{T} f(x, B)$.
\end{enumerate}
\end{Proposition}

\begin{proof}
\begin{enumerate}[label=\roman*.]
\item For any $O \in \mathcal{O}\big(q, \mathbb{R}\big)$, since $Z$ has a radial distribution, $Z \stackrel{\mathcal{L}}{\sim} OZ$, so that
	$$\mathcal{T} f\big(x, AO\big) = \mathbb{E}f\big(x+AOZ\big) = \mathbb{E}f\big(x+AZ\big) = \mathcal{T} f\big(x, A\big).$$

\item For any $x, y\in \mathbb{R}^d$ and $\lambda \in [0,1]$, the convexity of $f$ yields,
	\begin{align*}
	\mathcal{T} f\big(\lambda (x, A) + (1-\lambda) (y, B)\big) &= \mathbb{E}f\Big(\lambda(x+ AZ) + (1-\lambda)(y + BZ)  \Big)\\
	&\le \lambda \mathbb{E}f\big(x+AZ\big) + (1-\lambda) \mathbb{E} f\big(y+BZ\big)\\
	&= \lambda \mathcal{T} f\big(x, A\big) + (1- \lambda) \mathcal{T} f\big(y, B\big).
	\end{align*}

\item Note that if $A \preceq B$ then Proposition \ref{gen_radial_distri} implies $AZ \preceq_{cvx} BZ$. Thus using the convexity of $f(x + \cdot)$ (owing to the convexity of $f$), one has
$$\mathcal{T} f\big(x, A\big) = \mathbb{E}f\big(x+AZ\big))\le \mathbb{E}f\big(x+BZ\big) \le \mathcal{T} f\big(x, B\big).$$
\end{enumerate}
\end{proof}
\end{document}